\documentclass[sigconf, nonacm]{acmart}
\usepackage{graphicx, dblfloatfix}
\usepackage{balance}
\usepackage{amsmath,amsfonts}
\usepackage{textcomp}
\usepackage{booktabs}
\usepackage{xcolor}
\usepackage{tabularx}
\usepackage{array}
\usepackage{multirow}
\usepackage{hyperref}
\usepackage{url}
\usepackage{siunitx}
\usepackage{algorithm}
\usepackage[noend]{algpseudocode}
\usepackage{dsfont}
\usepackage[bottom]{footmisc}
\usepackage{subcaption}
\usepackage{bm}
\usepackage[inline]{enumitem}
\usepackage{xparse}
\usepackage[export]{adjustbox}

\DeclareMathOperator{\rreal}{r}
\DeclareMathOperator{\synth}{s}

\newcommand{\x}{\mathbf{x}}
\newcommand{\xh}{\mathbf{\hat{x}}}
\newcommand{\z}{\mathbf{z}}

\newcommand{\M}{\mathcal{M}}
\newcommand{\y}{\mathbf{y}}

\newcommand{\real}{\x^1}
\newcommand{\fake}{\x^0}

\newcommand{\para}[1]{\smallskip \noindent\textit{#1.}}

\setcopyright{acmcopyright}
\copyrightyear{2022}
\acmYear{2022}
\acmDOI{XX.XXXX/XXXXXXX.XXXXXXX}

\acmConference[CIKM '22]{2022 International Conference on Information and Knowledge Management (CIKM'22)}{October 17-22, 2022}{Atlanta, GA, USA}
\acmBooktitle{2022 International Conference on Information and Knowledge Management (CIKM '22), October 17-22, 2021, Atlanta, GA, USA}
\acmPrice{15.00}
\acmISBN{978-1-4503-9999-9/18/06}

\begin{document}

\title{GeoPointGAN: Synthetic Spatial Data with Local Label Differential Privacy}

\author{Teddy Cunningham}
\authornote{Both authors contributed equally to this research.}
\affiliation{%
  \institution{University of Warwick}
  \city{Coventry}
  \country{UK}
}
\email{teddy.cunningham@warwick.ac.uk}

\author{Konstantin Klemmer}
\authornotemark[1]
\affiliation{%
  \institution{University of Warwick}
  \city{Coventry}
  \country{UK}
}
\email{k.klemmer@warwick.ac.uk}

\author{Hongkai Wen}
\affiliation{%
  \institution{University of Warwick}
    \city{Coventry}
\country{UK}
}
\email{hongkai.wen@warwick.ac.uk}

\author{Hakan Ferhatosmanoglu}
\affiliation{%
  \institution{University of Warwick}
    \city{Coventry}
    \country{UK}
}
\email{hakan.f@warwick.ac.uk}

\renewcommand{\shortauthors}{Cunningham and Klemmer, et al.}

\begin{abstract}
    Synthetic data generation is a fundamental task for many data management and data science applications.
    Spatial data is of particular interest, and its sensitive nature often leads to privacy concerns.
    We introduce GeoPointGAN, a novel GAN-based solution for generating synthetic spatial point datasets with high utility and strong individual level privacy guarantees.  
    GeoPointGAN's architecture includes a novel point transformation generator that learns to project randomly generated point co-ordinates into meaningful synthetic co-ordinates that capture both microscopic (e.g., junctions, squares) and macroscopic (e.g., parks, lakes) geographic features.
    We provide our privacy guarantees through label local differential privacy, which is more practical than traditional local differential privacy. 
    We seamlessly integrate this level of privacy into GeoPointGAN by augmenting the discriminator to the point level and implementing a randomized response-based mechanism that flips the labels associated with the `real' and `fake' points used in training.
    Extensive experiments show that GeoPointGAN significantly outperforms recent solutions, improving by up to 10 times compared to the most competitive baseline.
    We also evaluate GeoPointGAN using range, hotspot, and facility location queries, which confirm the practical effectiveness of GeoPointGAN for privacy-preserving querying.
    The results illustrate that a strong level of privacy is achieved with little-to-no adverse utility cost, which we explain through the generalization and regularization effects that are realized by flipping the labels of the data during training. 
\end{abstract}


\maketitle

\section{Introduction}
\label{s:intro}

Generating synthetic datasets of high utility is a fundamental challenge in many data management tasks, such as private data sharing~\cite[e.g.,][]{Machanavajjhala2008, Zhang2017, Ge2021, Cunningham2021a, Cunningham2021}, 
benchmark generation~\cite[e.g.,][]{Gu2015, Ghazal2013}, and database management system testing~\cite{Bruno2005, Binnig2007, Torlak2012}.

A key priority when generating synthetic data is preserving the privacy of sensitive information while maintaining high utility.
Differential privacy (DP) and its local variant, LDP, have become the de facto privacy standards for synthetic data generation owing to their mathematically rigorous privacy guarantees, and several studies have tackled the issue of generating private synthetic datasets~\cite[e.g.,][]{Zhang2021, Huang2019, Cunningham2021}.
Whereas the centralized setting of DP relies on a trusted aggregator, LDP offers a stronger degree of privacy by allowing users to perturb their data \emph{before} sharing it with the aggregator.  

In this paper, we develop a locally private machine learning-based solution for generating synthetic data from real data, with the aim to preserve the characteristics of the original real data faithfully.
We focus on spatial point data as the proliferation of mobile technologies and location-based services has made (private) spatial data increasingly valuable to data scientists, companies, and researchers.
Spatial datasets are also used in several societally important fields, such as ecology~\cite{Velazquez2016}, geology~\cite{Zuo2009}, and epidemiology~\cite{Gatrell1996}, many of which need to contend with the need to preserve the privacy of the data subjects (e.g., animals from being poached illegally, individuals being identified from contact tracing apps).

However, when dealing with spatial data, the traditional form of LDP can be unnecessarily restrictive in terms of how it deals with sensitive data, as it normally adopts an ``all-or-nothing'' approach in which all data needs to be perturbed~\cite{Malek2021}.
This affects the utility of the synthetic data for common location analytics tasks.
Label privacy~\cite{Chaudhuri2011} provides a more practical means for achieving the necessary privacy protection without sacrificing utility.
It is based on the notion that the features of a point are public (and so do not need to be perturbed), whereas the label associated with the data is private (and so does need to be perturbed).
Applied to the setting of spatial data, label privacy leads to the following idea: as all location information is public knowledge, it is only a person's \textit{association} with a particular location point at a particular time that is private and in need of perturbation.
Consequently, we combine the idea of label privacy with LDP and utilize label-LDP to provide sufficient privacy protections when generating synthetic spatial data.

In theory, GANs offer a general purpose solution for the data generation problem due to their objective of learning the optimal functional mapping from some random noise input to a faithful representation of real data~\cite{Goodfellow2014}.  
However, most existing GAN architectures for point data stem from computer vision and are designed for capturing simplified continuous shapes and meshes. 
As such, they are not optimized to handle complex spatial patterns observed in the real world, as our experiments show. 
Furthermore, they do not support local privacy and are unsuitable for generating large-scale spatial datasets.
Our solution, GeoPointGAN, addresses these limitations through several key technical novelties in learning the multi-scale partitioned structures of spatial point data. 
GeoPointGAN's architecture uses extended PointNets in both the generator and discriminator and, rather than sampling from a lower-dimensional latent vector, the generator ingests randomly generated pseudo-co-ordinates and learns a transformation to generate realistic outputs.
The discriminator provides outputs on the point level, rather than the batch level, which allows for highly localized training and the seamless integration of local privacy mechanisms. 

Given that they operate with data with `real' and `fake' labels, GANs present an intuitive setting for implementing label-LDP.
GeoPointGAN incorporates label-LDP through a randomized response mechanism that flips the labels provided to the discriminator, thereby providing plausible deniability to each individual's \textit{association} with a location.
We also outline how, as labels are only flipped once, label-LDP does not suffer from the same vulnerability of traditional LDP in which locations can be revealed through repeated querying.
Beyond its privatization properties, label flipping also has potential generalization and regularization effects on the model performance, which we contextualize with related literature.  
In most settings, local privacy mechanisms are known to introduce a lot of noise and typically a large database is needed to generate useful output.
But, as we demonstrate in our experiments, incorporating label-LDP into our training process has negligible effects on model performance, which indicates that GeoPointGAN is effective at minimizing the impacts of the added noise.

We evaluate GeoPointGAN against three recent GAN-based approaches using four real-world datasets, each of which exhibit different characteristics.
The first set of experiments show that GeoPointGAN significantly outperforms existing GAN-based methods in generating spatial point data, improving by up to 10 times compared to the most competitive baseline.
In the second set of experiments, we use the synthetic data generated by GeoPointGAN to answer location analytics queries, namely range, hotspot, and facility location queries.
GeoPointGAN performs excellently in most settings and the synthetic data obtains very similar answers to the queries as the real data does.
Interestingly, in some settings, privatized GeoPointGANs perform slightly better than non-private GeoPointGANs, which supports our hypothesis that label flipping can help to realize generalization and regularization benefits.

Our strong results confirm that GeoPointGAN is a robust method for generating practical private synthetic data that can be a worthy substitute for non-private real data for many data science tasks, and it can also be used to accurately answer other queries, such as proximity queries and clustering-based analysis.
Finally, generating datasets with privacy guarantees motivates further downstream applications.
For example, by harnessing the ubiquity of mobile devices, open data portals can be populated with the private mobility patterns of millions of individuals, with little overhead cost due to the distributed nature of the data collection.
\section{Related Work}
\label{s:related-work}

\para{GANs}
GANs have been utilized for a range of data types, including image data~\cite{Goodfellow2014}, audio streams~\cite{Akbari2018}, text data~\cite{Chen2018}, traffic patterns~\cite{Zhang2020}, and gene expressions~\cite{Dizaji2018}. 
In the geospatial domain, GANs have been used for generating digital elevation maps~\cite{Klemmer2021a} and global surface temperatures~\cite{Klemmer2021d}.
Existing GAN approaches for continuous spatial point co-ordinates mostly deal with point clouds that are simplified, continuous representations of shapes and surfaces. 
The first GAN tailored to point clouds (r-GAN)~\cite{Achlioptas2018} builds on advances in processing point clouds in neural networks, most notably PointNet~\cite{Qi2017}. 
Generating faithful shapes based on point cloud datasets, such as ShapeNet~\cite{Chang2015}, is an active research challenge~\cite[e.g.,][]{Li2018,Shu2019,Gal2020}.
Further studies have utilized GANs for point cloud upsampling~\cite{Li2019a}, shape completion~\cite{Sarmad2019}, or against adversarial attacks~\cite{Zhou2020}. 
Their applications to real-world data have been limited thus far, with a few recent exceptions such as an application to Lidar data~\cite{Caccia2019}.

Geospatial datasets have different characteristics, which means their point patterns are very different from the point clouds that describe shapes and meshes. 
They do not have a mesh-like structure, they are typically more complex and noisy, and may be governed by underlying dynamics such as self-excitement. 
Few studies have tackled this class of data using GANs:~\citet{Xiao2017} use Wasserstein GANs to learn temporal (one-dimensional) point processes, while~\citet{Klemmer2019a} learn conditional GANs contextualized by the co-ordinates of continuous spatial point data. 
However, these works provide no intuition for point transformations or for producing new spatial point patterns similar to the input. 
The challenging nature of generating spatial point patterns and the lack of existing work addressing this problem help to motivate our work.

\para{Private GANs}
In recent years, there have been a number of studies proposing differentially private GANs, including DPGAN~\cite{Xie2018}, DP-CGAN~\cite{Torkzadehmahani2019}, PATE-GAN~\cite{Yoon2018}, and the work of~\citet{Frigerio2019}, which extends DPGAN to continuous, discrete, and time series data. 
Existing private GANs have focused on other specific domains, such as medical~\cite{Yoon2018, Beaulieu-Jones2019, Torfi2020}, image~\cite{Torkzadehmahani2019}, or time series data~\cite{Wang2020}, as opposed to spatial data.
Furthermore, all existing private GANs use centralized DP, normally by clipping and adding noise to the gradient during training~\cite[][]{Xie2018, Frigerio2019} or applying existing private frameworks~\cite[e.g.,][]{Yoon2018, Augenstein2019}.
This different privacy setting means we cannot compare them to our work.

\para{Location Privacy}
Both DP and LDP have increasingly been applied with location-specific variants, such as geoindistinguishability~\cite{Andres2013}.
Most work has focused on data publication with DP~\cite[e.g.,][]{Mohammed2011, Chen2012, Acs2014, Cormode2012, Xiao2015, Ghane2018, Gursoy2018}, as opposed to data synthesis and LDP.  
Some other work exists on private trajectory synthesis and publication ~\cite[e.g.,][]{He2015, Gursoy2020, Qu2020}, but recently proposed solutions, in both the centralized \cite[e.g.,][]{He2015, Gursoy2020} and local settings \cite[]{Cunningham2021a}, all possess a common limitation.
They all produce outputs that correspond to arbitrary grid cells or places of interest, whereas we generate co-ordinate data (i.e., the same form as the input data).  
While one could extend these solutions to generate individual points (e.g., by using uniform sampling), \citet{Cunningham2021} show that adapting existing solutions in this way fails to produce high-quality synthetic spatial point data.
Their purpose-built solution is designed for the centralized DP setting, which leaves generating high-quality spatial point data in local privacy settings as an important, yet unaddressed, challenge.
We note that there has been some work on location data in local settings (i.e., LDP and variants thereof). 
\citet{Chen2016} use personalized LDP for spatial data aggregation, \citet{Xiong2019} focus on continuous location sharing using randomized response, and \citet{Cunningham2021a} publish LDP-compliant sequences of places of interest.
However, extensions of these works to our setting are unviable owing to their fundamentally different problem and/or privacy settings.

\para{Label DP}
Label-DP was formally introduced by~\citet{Chaudhuri2011} and has since been the focus of several studies~\cite{Wang2019b, Ghazi2021, Esfandiari2021, Yuan2021, Malek2021}.
All of these works are based on the same premise as our work: only the labels attached to data are sensitive, with the data itself being non-sensitive.
Despite this work, almost all prior work has been in the centralized setting; only~\citet{BusaFekete2021} consider the local setting.
Hence, we are the the first to apply label-LDP to GANs to generate privacy-preserving spatial point data with LDP-style guarantees.

\section{Method}
\label{s:method}

\begin{figure*}[t]
    \centering
    \includegraphics[width = 0.93\textwidth]{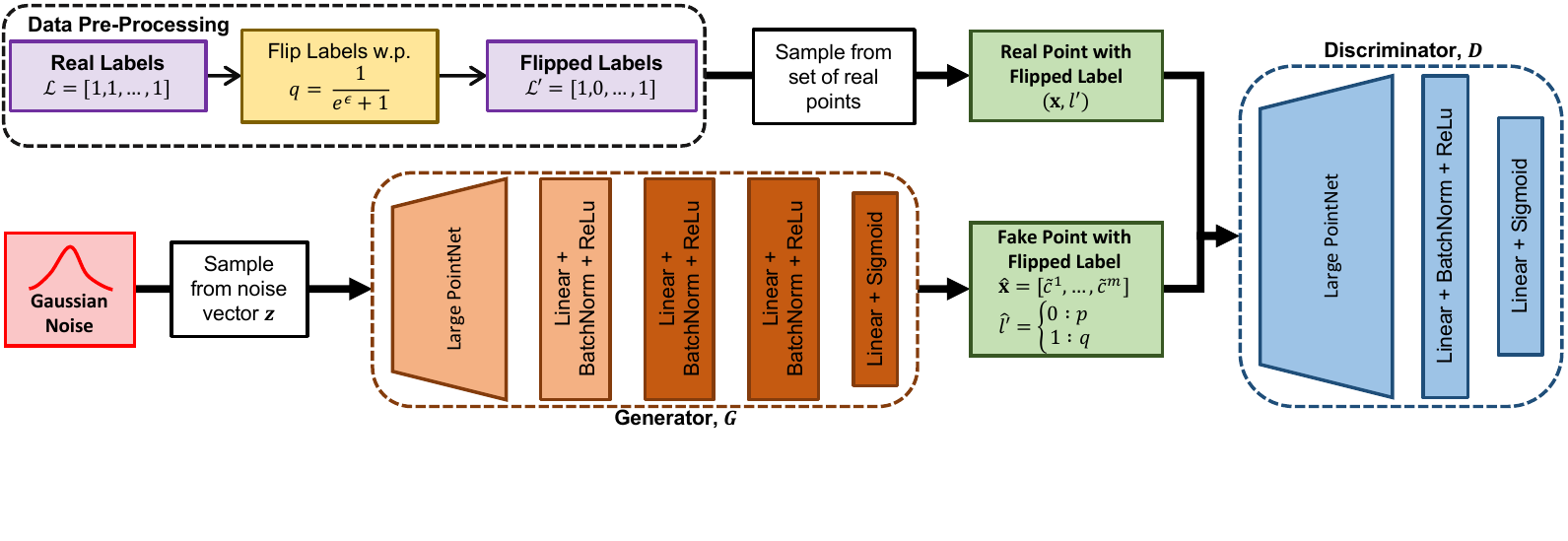}
    \vspace{-1.5cm}
    \caption{GeoPointGAN modeling pipeline including the label-LDP mechanism} 
    \label{fig:architecture}
    \vspace{-.2cm}
\end{figure*}

Before introducing GeoPointGAN, we first present some background on GANs, and discuss the privacy setting we follow in our work.

\subsection{GANs}
\label{sss:gans}
GANs are a family of models that seek to learn the data generating process of observed data $\x \sim p_{data}(\x)$. 
Learning is facilitated through two networks: a generator $G$, and a discriminator $D$.  
The generator $G(\z,\Theta_{G})$ with parameters $\Theta_{G}$ maps a random noise input $\mathbf{z} \sim p_{\z}(\z)$ to the feature space of the real data $\x$ so that $G:\z \xrightarrow{} \x$. 
The discriminator $D([\x,\xh], \Theta_{D})$, parameterized by $\Theta_{D}$, then attempts to distinguish real data $\x$ from synthetic samples $\xh$, so that $D: [\x,\xh] \xrightarrow{} \{0,1\}$, with 0 and 1 denoting labels of `fake' and `real' data points respectively. 
The learning process follows from a min-max game between $G$ and $D$, which is given as $V(D,G)$:
\begin{equation}
\begin{split}
\min_G \max_D V(D,G) = & \mathbb{E}_{\x \sim p_{data}(\x)} \bigl[\log D(\x)\bigr] + \\
& \mathbb{E}_{\z \sim p_{\z}(\z)} \bigl[\log (1 - D(G(\z)))\bigr]
\end{split}
\end{equation}

In our case, the input feature vector $\x \in \mathbb{R}^m$ represents the spatial co-ordinates of a point in $m$-dimensional space. 

\subsection{Privacy Setting}
\label{ss:privacy-setting}

\subsubsection{Local Differential Privacy}
\label{sss:ldp}
The centralized model of DP~\cite{Dwork2006} offers strong privacy guarantees through a level of plausible deniability but assumes that the data aggregator can be trusted, which may not always be appropriate, especially when highly sensitive data (e.g., location data) is concerned.  
Therefore, we focus on the local setting as it achieves a stronger level of privacy by allowing users to perturb the data before sharing it with the data aggregator.

\begin{definition}[$\epsilon$-local differential privacy~\cite{Duchi2013}]
A randomized mechanism $\M$ satisfies $\epsilon$-local differential privacy if, for any two inputs $\x_i, \x_j$ and output $\y$:
\begin{equation}
\textstyle
    \Pr[\M(\x_i) = \y] \leq e^\epsilon \times \Pr[\M(\x_j) = \y]
    \label{eq:ldp}
\end{equation}
\end{definition}
\noindent where $\epsilon$ is the privacy budget that controls the level of privacy protection.
The intuition with LDP is that, given the output $y$, an adversary cannot (with high confidence) identify the input value.

One important property of LDP is post-processing.
Formally, for any two mechanisms $\M_1$ and $\M_2$, if $\M_1$ satisfies $\epsilon$-LDP, then the composition $\M_2(\M_1(\cdot))$ satisfies $\epsilon$-LDP (regardless of whether $\M_2$ satisfies $\epsilon$-LDP itself) \cite{DelRey2020}.
In a practical sense, the post-processing property allows the output of an LDP mechanism to be used and manipulated infinitely, without affecting its privacy guarantee, as long as there is no further interaction with the private data.

LDP can be achieved by using randomized response~\cite{Warner1965} in which a user reports their true data with probability $p$, and reports an alternative response with probability $q$.
LDP attaches strong privacy guarantees to randomized response outputs. 
To satisfy LDP in the generalized randomized response setting, the probability a user reports true information is $p = \frac{e^\epsilon}{e^\epsilon + d - 1}$, where $d$ is the size of the output set.  
Hence, the probability of reporting any other single output is $q = \frac{1}{e^\epsilon + d - 1}$. 

\subsubsection{Label Local Differential Privacy}
\label{sss:label-ldp}
The notion of LDP is extended to label-LDP if we consider each feature vector $\x_i$ to have a label $l_i \in \mathcal{L}$, together denoted as $(\x_i, l_i)$.
\begin{definition}[$\epsilon$-label-LDP]
A randomized mechanism $\M$ satisfies $\epsilon$-label-LDP if, for any labeled feature vector $(\x, l)$ with the input labels $l_i, l_j \in \mathcal{L}$ and output label $l_k \in \mathcal{L}$:
\begin{equation}
\textstyle
    \Pr[\M((\x, l_i)) = (\x, l_k)] \leq e^\epsilon \times \Pr[\M((\x, l_j)) = (\x, l_k)]
    \label{eq:label-ldp}
\end{equation}
\end{definition}
It is intuitive to observe that label-LDP possesses the same post-processing property as traditional LDP, and that randomized response can be used to ensure label-LDP.

As explained in Section \ref{s:intro}, label-LDP provides more practical, yet sufficiently private, protection to data by only perturbing the label attached to a feature vector.
This is appropriate for our problem given that we deem information about locations in a region to be sufficiently public with only one's \textit{association} with a location being private information.
From Definition 3.2, the intuition is that an adversary cannot (with high confidence) identify whether the person was at the reported location or not.

We assume that the dataset of real points covers a sufficiently large proportion of the domain in which fake points can be generated.
This is to prevent the labels being sufficiently correlated with features, which would undermine privacy~\cite{BusaFekete2021}.
For example, if the ratio between land area and total area was too small, many fake points could be generated in nonsensical locations (e.g., oceans), which would allow an adversary to identify fake points and, by extension, real points.
While this assumption does not affect the mechanism in any way, we impose this soft constraint as an extra layer of protection against privacy leakage.

Finally, although one could na\"{i}vely apply randomized response directly to the real data, this limits the dataset size to approximately $\tfrac{Ne^\epsilon}{e^\epsilon + 1}$, which heavily limits the range of analytics tasks for which the private data can be used.
Hence, a more flexible solution that can generate datasets of any size is necessary.

\subsubsection{Examples}
\label{sss:example}
To further illustrate the use of label-LDP in our setting and its advantage over traditional LDP, consider the following two example scenarios.
In both examples, although having the original locations with perturbed labels is useful in itself, the samples may not be representative and, in any case, they will be noisy, owing to perturbation.
As such, being able to generate datasets of any size based on the original distribution is important, and gives end users more flexibility.

For our first example, consider that a city's government wants to know the distribution of its residents' locations at 10am.
This might help to determine the approximate proportion of people working from home, which is helpful for managing working conditions or reducing disease spread.
Given that the government will know every resident's address (e.g., through voter rolls or council tax information), this information is non-private.
The private element is where each person is at 10am, and residents will want a degree of plausible deniability, which is provided by label-LDP.
In comparison, traditional LDP is unnecessarily restrictive here as it would require the perturbation of the location of each resident at 10am, even if their home address is known.

Moreover, repeated use of traditional LDP (through perturbation of the location) would eventually reveal the true location. 
While continuous data sharing is a common need for many real-life applications, traditional LDP approaches do not address this most practical setting.
Whereas, with label-LDP, as the location itself is deemed to be non-private, each daily report can be considered to be independent, and so repeated querying does not degrade the overall privacy level provided to each user.
This further motivates label-LDP as a more robust privacy setting in many practical cases.

Our second example differs from the first as the government no longer has a plausible location for each individual for which they want to ask a yes-no question.
Imagine a town with 100,000 people, all of which are asked to privately share their location at 8pm, which is also the time at which a large number of residents are attending a concert.
With traditional LDP, the locations of each individual would need to be perturbed, which would induce a large amount of noise into the dataset and greatly affect utility.
However, public knowledge (e.g., news sources) would indicate that many people are at the concert and so we should expect a large number of reported locations to be associated with the concert.
As such, label-LDP still gives each concert-goer a degree of plausible deniability regarding their presence at the concert, while preserving the overall popularity distribution, which is publicly known (to some extent).

\subsection{Model Architecture}
\label{ss:geopointgan}
Existing GAN architectures for point data, mostly stemming from computer vision research, aim to capture object shape point clouds or mesh-structured data. 
As such, they are designed to model simple shapes and object outlines.  
Spatial point patterns, such as location data from mobile devices, typically have a noisy, multi-scale partitioned structure, and may cover the whole observational area, rather than having clear outlines.
For example, while the point cloud of a chair can roughly be segmented into six elements (four legs, seat, back), spatial point patterns in the real world can consist of hundreds of intricate macroscopic (e.g., terrain, cities) and microscopic (e.g., roads, junctions) elements.
GeoPointGAN includes several novel approaches to address the challenges of generating this data.
Its data-processing pipeline and architecture, outlined in Figure~\ref{fig:architecture}, consist of three main components, which we explain next.

\subsubsection{Generator}
\label{sss:generator}
GeoPointGAN samples a noise vector of the \textit{same} dimensionality as the desired output.
This is equivalent to sampling $B$ random points $\z = [\tilde{c}_{z}^{1},...,\tilde{c}_{z}^{m}]$ with co-ordinates $c$ in the same space $m$ as a real point $\x = [c^{1},...,c^{m}]$. 
This is in contrast to traditional GAN approaches, which sample a Gaussian noise vector $\mathbf{z}$ from a lower-dimensional latent space. 
Rather than learning a model that `upsamples' from a low-dimensional latent space to a higher-dimensional output space, we thus aim to learn a model that transforms data from an $m$-dimensional latent space to a meaningful representation in the same $m$-dimensional space.  

We use this noise sampling strategy to design a novel PointNet-based generator. 
PointNet~\cite{Qi2017} was originally devised for classification and segmentation of raw point clouds. 
While it has been used as the basis for GAN discriminators before 
\cite[see][]{Achlioptas2018}, we propose the first GAN that utilizes PointNet in the generator. 
Particularly, its ability in providing transformation invariant properties for unordered data is desirable for spatial point generation. 
This is achieved by running the input through symmetric functions (e.g., max pooling operators) to compute global point set features. 
This step is followed by a segmentation network that combines the global information (e.g., city boundaries, rivers) with local, point-wise information (e.g., roads, junctions) to learn a combined representation. 
Lastly, a spatial transformer network (STN)~\cite{Jaderberg2015} aligns the learned global and local point set features with the output space.
We find that the traditional STN architecture is unable to resolve the complexities of spatial point datasets sufficiently, owing to the shallowness of the neural networks deployed. 
In particular, while macroscopic structures (e.g., coastlines) can be captured reliably, the STN is incapable at learning small-scale patterns (e.g., minor roads).
Consequently, we extend the STN such that we have five one-dimensional convolutional layers, with four fully connected layers, adding batch normalization between every layer and the ReLu function (except the last layer). 
Altogether, we refer to this altered PointNet as `Large PointNet'.

As a final step, our generator takes the transformation invariant, aligned features of Large PointNet and projects them into the $m$-dimensional output space using four fully connected layers. 
This is similar to the prediction head for point segmentation, only that we produce $m$ features per point (our synthetic co-ordinates), rather than one. 
These design choices are informed by extensive testing and are validated in our experiments.

\subsubsection{Discriminator}
\label{sss:discriminator}
Our discriminator architecture is inspired by~\citet{Achlioptas2018}, but comes with some fundamental technical improvements and a critical change that lets us incorporate our label-LDP mechanism.
First, we run points through our `Large PointNet' module. 
This balances the capacity for learning point set representations between the generator and discriminator, allowing for an evenly matched min-max game. 
The discriminator's prediction head consists of two fully connected layers and a sigmoid activation. 
Second, we alter the last fully connected layer to produce predictions not on the batch, but rather the point level.
That is, the discriminator's task is thus to determine whether each \textit{individual} point is real or fake, as opposed to each batch of points.
Constructing the discriminator in this way allows us to seamlessly incorporate a localized privacy mechanism into model training.

\subsubsection{Privacy Mechanism}
\label{sss:privacy-mechanism}
We integrate point-level privacy guarantees into GeoPointGAN by probabilistically flipping the labels of the real and fake points (using randomized response) before showing the data samples to the discriminator.
Whereas LDP would require perturbation of each co-ordinate point (using, say, the Laplace or exponential mechanism), with label-LDP, perturbing the label is sufficient.
In our setting, we have two (pseudo-)labels -- `real' and `fake' -- which means that $d = 2$ and that flipping can be conducted using biased coin tosses in which the probability that a label's true status is maintained is $p$, and the probability that a label's status is flipped is $q = 1-p$.
The labels of real points are flipped on users' devices during data pre-processing, which ensures that the central agent (i.e., the GAN networks) never has access to this information.
Should a real point be sampled several times throughout training, it will always have the same (flipped) label. 
From Definition 3.2, the intuition follows that the discriminator cannot determine (with high probability) that a point with a real label is actually real, or whether it is a fake point masquerading as a real one (or vice versa).

In a practical setting, all training is conducted by a central agent (who can be trusted or untrusted) on a remote server.
Individual data is collected through mobile devices and each individual is responsible for flipping the label associated with their location.
Hence, the only data transferred from this device is the location and the flipped label, which means that no central agent can definitively determine the true label with absolute certainty.
Importantly, the discriminator does not know which points are generated by the generator, and which points are transmitted to the server by users; it can only distinguish points based on their perturbed labels.
In summary, the discriminator has no way of (definitively) knowing whether any one point is real with a real label, real with a fake label, fake with a real label, or fake with a fake label.

\subsubsection{Effects on Training and Generalization}
\label{sss:effects}

A label flipping approach such as ours does not necessarily reduce model performance, but can even have beneficial effects. 
In predictive models, randomly flipping labels can act as a regularizer, preventing the model from overfitting and improving generalization~\cite{Xie2016}.
When working with noisy labels, label flipping can incorporate the uncertainty of the labels into the model~\cite{Nguyen2019}. 
There is a vast collection of literature that focuses on GAN regularization and robustness, and addresses problems such as limited data availability or generator-discriminator imbalance.  
Manipulating the (pseudo)-labels of GANs has proven to be a successful strategy to this end.
Specifically, adding noise to the labels or applying one-sided label smoothing have been shown to improve GAN training and these are common best practices~\cite{Salimans2016}.
Finally,~\citet{Jiang2021} provide a study that is closely related to our approach. 
The authors propose to feed the discriminator with fake data masquerading as real data (i.e., fake data with real labels). 
While this is proposed mainly as an augmentation strategy for sparse data environments, it is very similar to our label flipping approach, although they do not feed the discriminator real data with fake labels as we do. 
The authors also provide a theoretical intuition for training convergence and their approach.
Their proof highlights how a GAN trained with label flipping augmentation minimizes the Jensen-Shannon divergence between the (smoothed) real and synthetic data distributions and is, in theory, able to perfectly capture the data generating process.
Hence, we expect that private GeoPointGANs will (to a certain extent) perform as well as non-private GeoPointGANs (i.e., ones with no label flipping).

\subsection{Model Training}
\label{sss:pseudocode}
Algorithm \ref{alg:training} describes the training of GeoPointGAN.
Note that Lines 1-2 are conducted on user devices, although we include the steps here for completeness.
Before training, each real point is assigned the `real' label: $l_i = 1$ (Line~1). 
These labels are then flipped with probability $q$, where $q$ is controlled by the privacy budget, $\epsilon$ (Line~2).
Perturbed labels are denoted as $l'_i$.
We then initiate the training loop (Line~3).
At each training step, $B$ points are sampled from the real data \textit{without} replacement (so as to not oversample points from high-density areas)  (Line~4).  
$B$ random co-ordinates are also drawn from the noise prior $p_{\z}$ (Line~5) and transformed in $G$ to generate $B$ fake points: $\xh = G(\z)$.
The label of each fake point, $\hat{l}_i = 0$, is flipped with probability $q$ to obtain $\hat{l}'_i$ (Line~6).
Real points $(\x, l')$ and fake points $(\xh, \hat{l}')$ are then classified as real or fake by $D$, after which $D$ is updated using the optimizer (Line~7).
We then train $G$ by generating $B$ new fake points (Line~8), flipping their labels with probability $q$ (Line~9), and once more classifying them using $D$. 
$G$ is then updated using its optimizer (Lines 10). 
This concludes one training step. 
$D$ and $G$ continuously play this game, with $G$ getting better and better at generating synthetic data.
After $T$ training steps, the label-LDP generator is published (Line~11).


\begin{algorithm}[t]
\small
\caption{GeoPointGAN training}\label{alg:training}
\algrenewcommand\algorithmicindent{0.75em}%
	\begin{algorithmic}[1]
	    \Require{$\mathcal{D}$, $\epsilon$, $B$, $T$}
	    \State Assign real points real labels: $\{\x_1, ..., \x_n\} \rightarrow \{(\x_1, l_1), ..., (\x_n, l_n)\}$
	    \State Flip real labels with prob. $q = \frac{1}{e^{\epsilon}+1}$ to obtain $\{(\x_1, l'_1), ..., (\x_n, l'_n)\}$
		\For{$1$ \textbf{to} $T$}
		    \State Sample $B$ real points with flipped labels: $\{(\x_1, l'_n), ..., (\x_B, l'_B)\}$
		    \State Sample $B$ random co-ordinates $\{\z_1, ..., \z_B\}$ from noise prior $p_{\z}$
		    \State Flip fake point labels with probability $q$ to obtain $\{\hat{l}'_1, ..., \hat{l}'_n\}$ 
		    \State Update $D$ by ascending its stochastic gradient: 
		    \Statex $\nabla_{\Theta_D} \frac{1}{B} \sum_{i=1}^{B} \Big[\log D(\x_i,l'_{i}) + \log (1 - D(G(\z_i),\hat{l}'_{i})\Big] \nonumber$
		    \State Sample $B$ random co-ordinates $\{\z_1, ..., \z_B\}$ from noise prior $p_{\z}$
		    \State Flip fake point labels with probability $q$ to obtain $\{\hat{l}'_1, ..., \hat{l}'_n\}$ 
		    \State Update $G$ by ascending its stochastic gradient: \Statex $\nabla_{\Theta_G} \frac{1}{B} \sum_{i=1}^{B} \Big[\log (D(G(\z_i),\hat{l}'_{i})\Big] \nonumber$
		\EndFor
		\State \textbf{return} $G$
	\end{algorithmic}
\end{algorithm}

\subsection{Privacy Analysis}
\label{ss:privacy-analysis}
We now discuss some important aspects of the mechanism, from a privacy perspective.
First, we show the proposed label flipping approach satisfies $\epsilon$-label-LDP. 
\begin{theorem}
    GeoPointGAN satisfies $\epsilon$-label-LDP.
\end{theorem}
\begin{proof}
For ease of understanding, in this proof, we denote points with real and fake labels as $\real$ and $\fake$, respectively
The probability that a real label tells the discriminator that it is a real label is: $\Pr[\M(\real)\!=\!\real]\!=\!p\!=\!\frac{e^\epsilon}{e^\epsilon 
+ 1}$.  
Hence, the probability that a real label tells the discriminator that it is a fake label is $1\!-\!p$.  
That is, $\Pr[\M(\real)\!=\!\fake]\!=\!q\!=\!\frac{1}{e^\epsilon + 1}$.
Similarly, $\Pr[\M(\fake)\!=\!\fake]\!=\!p$ and $\Pr[\M(\fake)\!=\!\real]\!=\!q$.
From Equation \ref{eq:label-ldp}, we have:
\begin{equation*}
\textstyle
    \frac{\Pr[\M(\real) = \real]}{\Pr[\M(\fake) = \real]} = \frac{p}{q} = \frac{e^\epsilon}{e^\epsilon + 1} / \frac{1}{e^\epsilon + 1} = e^\epsilon
    \tag*{\qedhere}
\end{equation*}
\end{proof}

As discussed in Section \ref{ss:privacy-setting}, label-LDP has the same post-processing properties as LDP, which means that privatized data can be manipulated freely without affecting the privacy guarantee (as long as the true data is not `touched' again).
In our setting, the labels are perturbed by users before they are shown to the discriminator, and the true label is never used again.
This means that the entire training procedure operates under post-processing, and our privacy guarantee remains in tact throughout training.

Many DP mechanisms suffer when points are repeatedly sampled, which causes the privacy leakage to increase each time a point is sampled.
Our mechanism is designed such that these attacks are redundant as point labels are flipped once, and once only, before training begins.
That is, when $\x_i$ is sampled during any training step, it will always have the same perturbed label $l'_i$.
This means that there is no privacy leakage even if a point is sampled more than once during different training steps.
Note that, as we sample \textit{without} replacement, the same point cannot be sampled multiple times during a single training step.

\begin{figure*}[t]
    \centering
    \begin{subfigure}{\textwidth}
        \centering
        \includegraphics[scale=0.5]{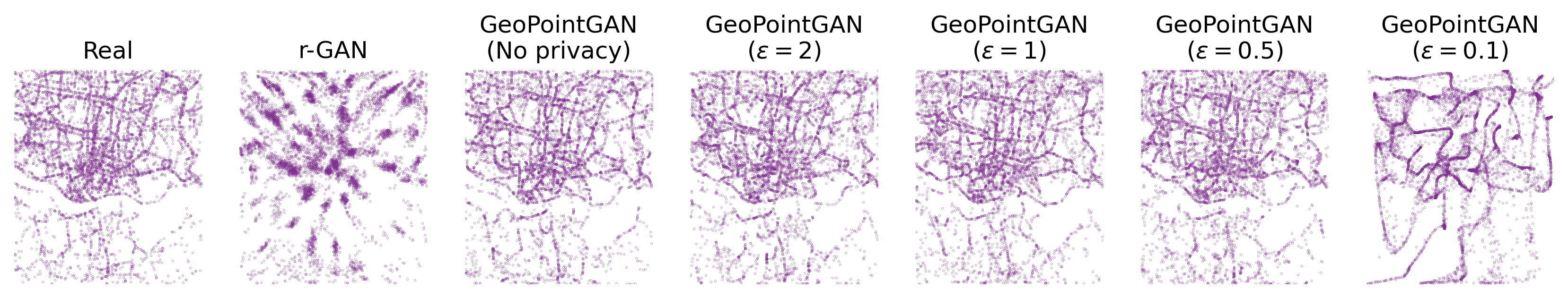}\vspace{-0.5em}
    \end{subfigure}
    \begin{subfigure}{\textwidth}
        \centering
        \includegraphics[scale=0.5]{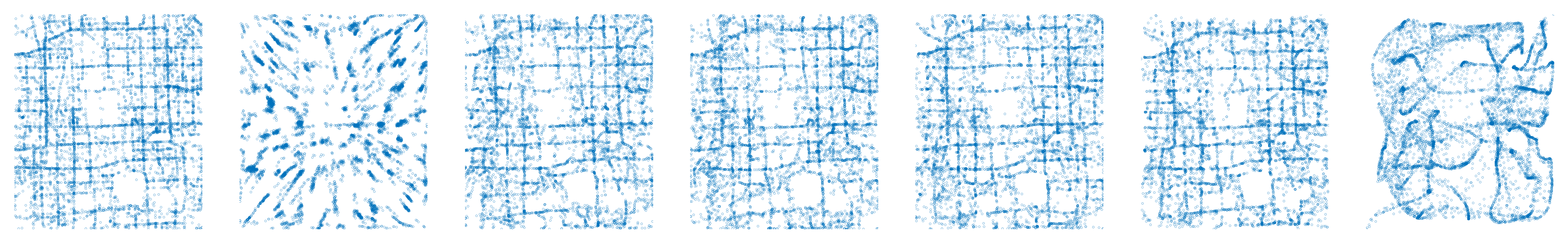}\vspace{-0.5em}
    \end{subfigure}
    \begin{subfigure}{\textwidth}
        \centering
        \includegraphics[scale=0.5]{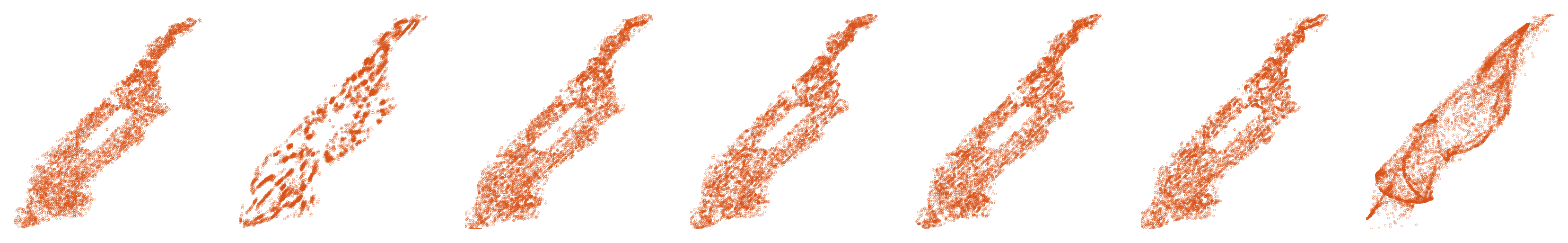}\vspace{-0.5em}
    \end{subfigure}
    \begin{subfigure}{\textwidth}
        \centering
        \includegraphics[scale=0.5]{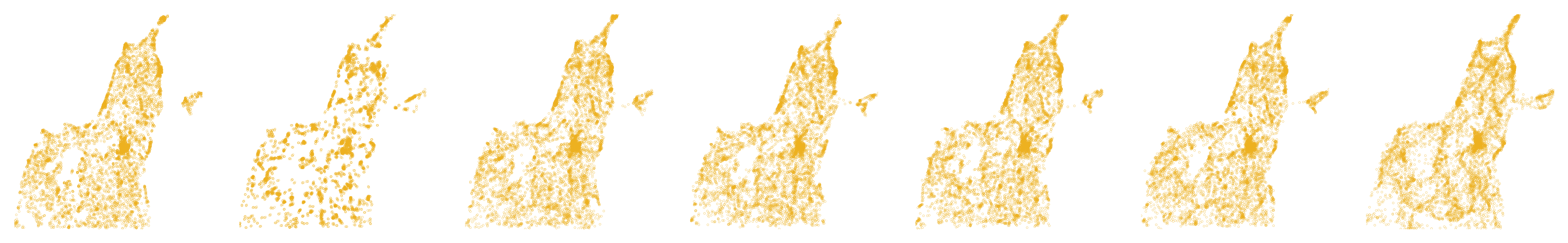}\vspace{-0.5em}
    \end{subfigure}
    \vspace{-.2cm}
    \caption{Sample plots of real and synthetic data;
    from top to bottom: Porto, Beijing, New York City and 3D Road}
    \label{fig:gen}
    \vspace{-.1in}
\end{figure*}

\section{GAN Evaluation}
\label{s:expts}
We evaluate GeoPointGAN in two parts using four real spatial datasets.
The first part, presented in this section, evaluates both non-private and private versions of GeoPointGAN using two fundamental point cloud and GAN metrics. 
We compare GeoPointGAN against three alternative GAN-based approaches. 
In the second part (Section \ref{s:application-queries}), we evaluate GeoPointGAN's practical query-based performance using three common location analytics tasks.

\subsection{Experiment Set-Up}
\label{ss:experiment-settings}

\para{Data}
All experiments use the following open-source, real-world datasets.
They exhibit a range of characteristics (e.g., alignment of the data with the road network, structure of the road network), which allow us to study GeoPointGAN in a variety of contexts.

We use two taxi trajectory datasets from Porto~\cite{Porto2015} and Beijing~\cite{Yuan2010, Yuan2011}.  
We extract the latitude and longitude co-ordinates from the raw data and, remove all points that fall within nonsensical geographic regions (e.g., bodies of water).  
Although points in the real data are linked, we remove this spatio-temporal linkage and consider each point individually.  
Note that doing this has no effects on privacy, as each point has its own privacy guarantee.  
The final Porto dataset contains 79,360 points over an area of 24.7km\textsuperscript{2}, and the Beijing dataset contains 158,260 points across a region with an area of 104.7km\textsuperscript{2}.

We also use 311 call data from New York City (NYC)~\cite{NYC2020}, filtering the dataset to only include data from Manhattan. 
The dataset has 163,220 points, each of which represents the location provided by the caller.
Unlike the other datasets, the New York data is closely aligned with the road network, which itself is grid-like and ordered.

Our final dataset -- `3D road' -- provides three-dimensional spatial co-ordinates (latitude, longitude, and altitude) of the road network in Jutland, Denmark \cite{Kaul2013}. 
The dataset comprises over 430,000 points, covering an area of 185 $\times$ 135 km\textsuperscript{2}. 

\para{Training Setting}
We follow a standardized training process.
At each training step, 7,500 points are randomly sampled from the real dataset. 
We use the Adam optimizer with decoupled weight decay~\cite{Loshchilov2019} and an initial learning rate of \num{4e-5}. 
The learning rate is decreased by a factor of 10 after 5,000, 50,000, and 90,000 training steps. 
We train 1,000 steps per epoch for a total of 100 epochs.
All training is conducted on a single RTX 2080 GPU. 
With this set-up, model training times do not exceed two hours.
Overall, GeoPointGAN, like r-GAN, experiences training that is reliable and consistent. 
At no point during any of our training runs do we experience mode collapse or exploding gradients.

\para{Benchmarks}
We compare GeoPointGAN against three state-of-the-art GANs.
The first -- r-GAN~\cite{Achlioptas2018} -- is the method that is most closely related to GeoPointGAN and it is designed to operate on raw point clouds (point co-ordinates). 
The other two baselines, tree-GAN~\cite{Shu2019} and PCGAN~\cite{Arshad2020}, are designed for graph-structured point clouds (e.g., meshes, shapes). 
All baselines are trained according to the configuration outlined by the original authors. 

Despite the development of other private GANs (e.g., DPGAN, PATEGAN), these works all use the centralized DP, which is fundamentally different from our label-LDP setting.
Similarly, although other private methods for synthetic spatial data generation exist, these methods also use different forms of privacy.
For example, \citet{Chen2016} use personalized LDP, and \citet{Cunningham2021} use centralized DP.
As our privacy setting is different from all of these works, any comparison between them is meaningless.

\para{Evaluation Metrics}
To evaluate GeoPointGAN's ability to preserve the underlying distribution of the real data, we use two widely used utility measures:  Chamfer distance (CD) and earth mover's distance (EMD).
As a continuous and pairwise smooth function, Chamfer distance is a well-established measure of point cloud distance.  
For a set of real points ($\mathcal{R}$) and synthetic points ($\mathcal{S}$), CD is defined as: 
\begin{equation}
\textstyle
\label{eq:chamfer}
    CD(\mathcal{R},\mathcal{S}) = \sum_{r \in \mathcal{R}} \min_{s \in \mathcal{S}} \|r-s\|^{2} + \sum_{s \in \mathcal{S}} \min_{r \in \mathcal{R}} \|r-s\|^{2}
\end{equation}
where $\|\cdot\|$ is some distance measure, and $r$ and $s$ are individual real and synthetic points, respectively.
We use normalized Euclidean distance in our evaluation.

EMD -- a common metric for evaluating GANs -- can be viewed as an optimization problem that seeks to transform one probability distribution into another while minimizing the cost of this operation. 
While the computational cost of obtaining the exact distance is too high for it to be used in deep learning algorithms (and hence approximations are used), we use its exact version as an evaluation measure. 
Defining $\phi: \mathcal{R} \rightarrow \mathcal{S}$ as a bijection, EMD is defined as:
\begin{equation}
\textstyle
    EMD(\mathcal{R},\mathcal{S}) = \min_{\phi: \mathcal{R} \rightarrow \mathcal{S}} \sum_{e \in \mathcal{R}} \| r - \phi (r)\|
\end{equation}

\subsection{Results}
\label{ss:results}
\subsubsection{Baseline Comparison}
\label{sss:baseline-comparison}
Figure \ref{fig:gen} depicts plots of the real data, alongside samples from GeoPointGAN and r-GAN.
We do not include visualizations for tree-GAN and PCGAN as they performed poorly.
While r-GAN succeeds at capturing macroscopic structures, such as the outline of Manhattan and Central Park, it lacks the capacity to model more intricate structures within the outline, such as individual streets or junctions.
On the other hand, GeoPointGAN is able to reproduce microscopic structures to a reasonable extent. 

We calculate mean CD and EMD values by taking 60 samples of 7,500 points each from the fully trained generators.
These values are shown in Table \ref{tab:cd-emd}. 
Note that, here, we use a non-private GeoPointGAN (i.e., $\epsilon = \infty$) to ensure a fair comparison with the baselines, which have no privacy mechanism.
The tree-based methods perform particularly poorly as they fail to learn the spatial data distribution.
This justifies the decision to use generative models that are capable at operating on raw point co-ordinates, rather than shapes or meshes.
GeoPointGAN offers substantial improvements of up to 70\% over r-GAN, with consistently better CD and EMD values, which reflects its ability to preserve microscopic features and generate accurate spatial point data.

\begin{table}[t]
\centering
\small
\addtolength{\tabcolsep}{-2.5pt}
\caption{Mean CD and EMD values}
\vspace{-.1in}
\label{tab:cd-emd}
    \begin{tabular}{c|cccc|cccc}
    \toprule
    \multirow{2}{*}{\textbf{Method}} & \multicolumn{4}{c|}{\textbf{Chamfer Distance}} & \multicolumn{4}{c}{\textbf{Earth Mover's Distance}} \\
    & \textbf{NYC} & \textbf{Por.} & \textbf{Beij.} & \textbf{3DR} & \textbf{NYC} & \textbf{Por.} & \textbf{Beij.} & \textbf{3DR} \\
    \midrule
    tree-GAN \cite{Shu2019} & 0.649 & 0.437 & 0.651 & 1.247 & 0.601 & 1.085 & 0.733 & 1.080 \\
    PCGAN \cite{Arshad2020} & 0.160 & 0.348 & 0.092 & 0.796 & 0.305 & 0.831 & 0.283 & 0.896 \\
    r-GAN \cite{Achlioptas2018} & 0.226 & 0.034 & 0.032 & 0.264 & 0.364 & 0.275 & 0.084 & 0.526 \\
    GeoPointGAN  & \textbf{0.014}	& \textbf{0.019} & \textbf{0.021} & \textbf{0.074} & \textbf{0.031} & \textbf{0.027}	& \textbf{0.032}	& \textbf{0.085} \\
    \bottomrule
    \end{tabular}
    \vspace{-.5cm}
\end{table}

\begin{figure*}[t]
    \centering
    \includegraphics[width=0.9\textwidth]{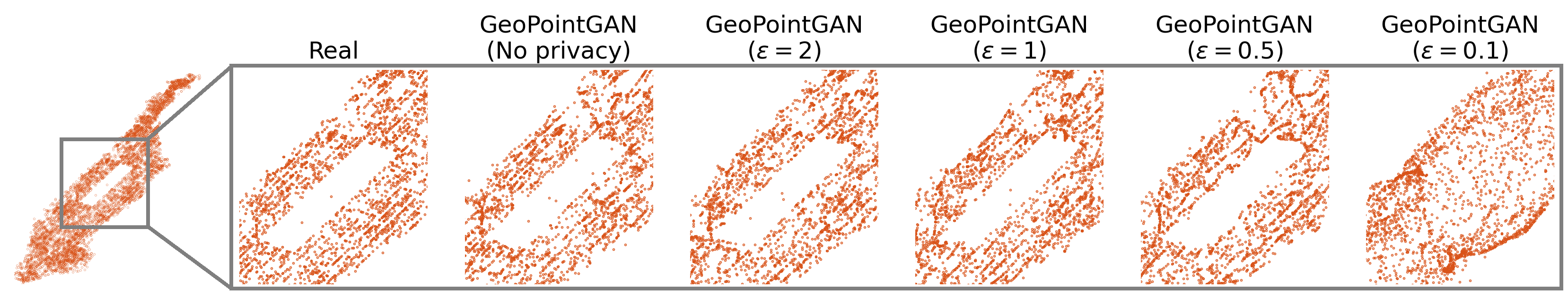}
    \vspace{-.2cm}
    \caption{Real and synthetic data for different privacy budgets; data for New York City, zoomed in on Central Park}
    \label{fig:zoomed-in}
    \vskip -0.1in
\end{figure*}
\begin{figure}[t]
    \centering
    \begin{subfigure}[b]{0.49\columnwidth}
        \centering
        \includegraphics[height = 2.2cm]{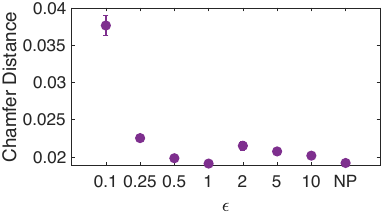}
        \vspace{-0.2cm}
        \caption{Porto}
        \label{fig:porto-cd}
    \end{subfigure}
    \hfill
    \begin{subfigure}[b]{0.49\columnwidth}
        \centering
        \includegraphics[height = 2.2cm]{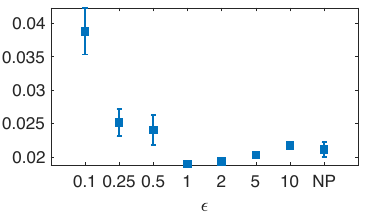}
        \vspace{-0.2cm}
        \caption{Beijing}
        \label{fig:beijing-cd}
    \end{subfigure}
    \\
    \begin{subfigure}[b]{0.49\columnwidth}
        \centering
        \includegraphics[height = 2.2cm]{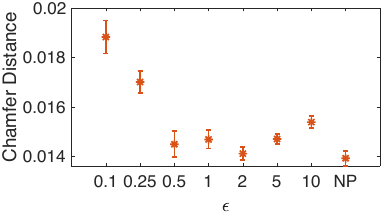}
        \vspace{-0.2cm}
        \caption{New York}
        \label{fig:nyc-cd}
    \end{subfigure}
    \hfill
    \begin{subfigure}[b]{0.49\columnwidth}
        \centering
        \includegraphics[height = 2.2cm]{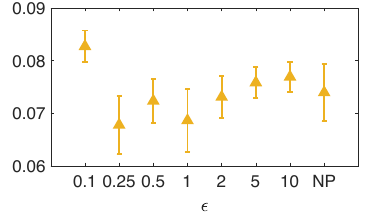}
        \vspace{-0.2cm}
        \caption{3D Road}
        \label{fig:3droad-cd}
    \end{subfigure}
    \vspace{-0.1in}
    \caption{Chamfer distance across different datasets}
    \label{fig:chamfer}
    \vspace{-0.4cm}
\end{figure}

\subsubsection{Effect of Privatization}
\label{sss:privatization}
To assess how the application of our label-LDP mechanism affects GeoPointGAN performance, with particular focus on the sensitivity of GeoPointGAN to the privacy budget, we again draw 60 samples of 7,500 points from the real and synthetic datasets.
We consider seven privacy budgets: $\epsilon = \{0.1, 0.25, 0.5, 1, 2, 5, 10\}$.
As none of the baselines are competitive with GeoPointGAN and they were not designed with privacy-preserving mechanisms in mind, we exclude them from this part of the evaluation.
Figure \ref{fig:gen} shows the effect of privatization on the generated data. 
We observe that lower privacy budgets have a negative impact on visual similarity.
Figure \ref{fig:zoomed-in} further illustrates how the preservation of macroscopic geographic features (e.g., Central Park in New York City) is affected by changing the privacy budget.

Figure \ref{fig:chamfer}, which shows the variation in CD as the privacy budget varies, highlights more interesting behavior.
First, as expected, a very low privacy budget results in poor utility with utility increasing as $\epsilon$ increases. 
However, when $\epsilon > 1$, the CD values start to increase, leading to (albeit subtle) U-shaped curves.
This indicates that our hypothesis that label flipping can aid performance is supported with empirical evidence.
Specifically, we can quantify the optimal degree of label flipping to be when $\epsilon \approx 1$ (i.e., $q \approx 27\%$).
In some cases, better utility is gained using \textit{private} GeoPointGANs (cf. Beijing), which further demonstrates the power of regularization through privatization.

Interestingly, the structure of the underlying data also appears to influence this behavior.
In New York, where the data is more closely aligned with a strict grid structure, the U-shape is more pronounced, indicating that the regularization properties of GeoPointGAN are more influential here.
Conversely, in Porto, where the true data (and road network) lacks a clear structure, the U-shape is more subtle and very high privacy budgets correct the curve downwards.
Finally, for 3D Road, we see large CD values and larger variations in these values, which suggests that the complexity and spatial extent of the dataset pushes GeoPointGAN to its limits.

\section{Data Analytics Tasks}
\label{s:application-queries}
We now evaluate GeoPointGAN using two popular data analytics tasks -- range and hotspot queries -- both of which are popular for location analytics.
We then apply our solution to the data-driven spatio-temporal task of facility location.
We assess the ability for GeoPointGAN to preserve the answers to these queries to demonstrate the applicability of our solution in a database and data science setting.
Given the poor qualitative and quantitative performance of the baselines, and their non-private nature, it is not meaningful to compare GeoPointGAN against them for these queries.
As such, the aim of this evaluation is to compare GeoPointGAN to the optimum (i.e., maximum similarity with the real data).

\begin{figure}[t]
    \centering
    \begin{subfigure}[b]{\columnwidth}
        \centering
        \includegraphics[height=0.5cm]{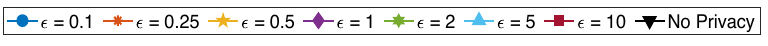}
    \end{subfigure}
    \\
    \begin{subfigure}[b]{0.34\columnwidth}
        \centering
        \includegraphics{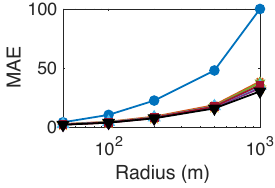}
        \vspace{-0.2cm}
        \caption{MAE -- Porto}
        \label{fig:porto-range}
    \end{subfigure}
    \hfill
    \begin{subfigure}[b]{0.32\columnwidth}
        \centering
        \includegraphics{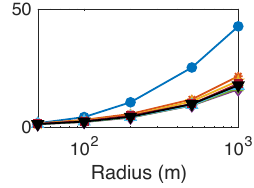}
        \vspace{-0.2cm}
        \caption{MAE -- Beijing}
        \label{fig:beijing-range}
    \end{subfigure}
    \hfill
    \begin{subfigure}[b]{0.32\columnwidth}
        \centering
        \includegraphics{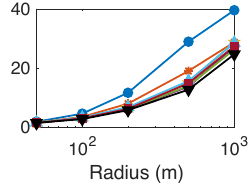}
        \vspace{-0.2cm}
        \caption{MAE -- New York}
        \label{fig:nyc-range}
    \end{subfigure}
    \\
        \begin{subfigure}[b]{0.34\columnwidth}
        \centering
        \includegraphics{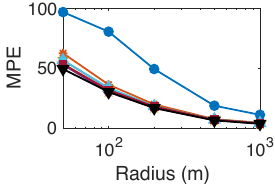}
        \vspace{-0.2cm}
        \caption{MPE -- Porto}
        \label{fig:porto-range}
    \end{subfigure}
    \hfill
    \begin{subfigure}[b]{0.32\columnwidth}
        \centering
        \includegraphics{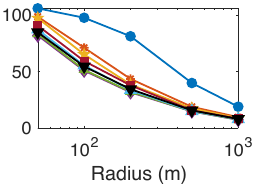}
        \vspace{-0.2cm}
        \caption{MPE -- Beijing}
        \label{fig:beijing-range}
    \end{subfigure}
    \hfill
    \begin{subfigure}[b]{0.32\columnwidth}
        \centering
        \includegraphics{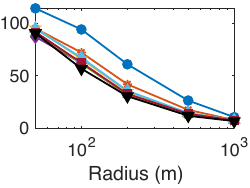}
        \vspace{-0.2cm}
        \caption{MPE -- New York}
        \label{fig:nyc-range}
    \end{subfigure}
    \vspace{-0.7cm}
    \caption{Variation in MAE/MPE as the query radius varies}
    \vspace{-.4cm}
    \label{fig:range}
\end{figure}

\subsection{Range Queries}
\label{ss:range}
Range queries are commonly used in databases, as well as a primitive for location analytics. 
For example, they can be used to quickly assess how many customers are potentially available to a business, or measure accessibility to key services within a certain time, such as schools, hospitals, or vaccination centers.
To assess this, we specify a set, $\mathcal{P}$, of 200 arbitrary places in each city to be the basis of our range queries, and these places are randomly selected from the set of nodes in each city's road network.
The extent of each range query is the circular region defined by the radius, $\rho$, and centered on place $\pi$.
For each synthetic sample, we answer the set of range queries for all $\pi \in \mathcal{P}$ and all values for $\rho$, and we quantify of error using
the mean absolute error (MAE) and mean percentage error (MPE).  
Having conducted range queries for all 60 synthetic data samples, we take the mean of the MAE and MPE values.
We consider the following $\rho$-values: \{50, 100, 200, 500, 1000\} meters.

Figure \ref{fig:range} shows the effect that $\rho$ has on the MAE/MPE of the query answer.
As expected, MAE increases as $\rho$ increases, although MPE decreases; both of these trends are acceptable when considered together.
As a particularly impressive outcome, we see that synthetic data generated with higher privacy budgets (i.e., $\epsilon \geq 0.25$) performs as well, and sometimes better, than the data generated using the non-private version of GeoPointGAN. 
Close inspection indicates that the privatized GeoPointGAN performs best when $\epsilon \approx 1$, which is concordant with the findings in Section \ref{sss:privatization}.

\subsection{Hotspot Analysis}
\label{ss:hotspots}
Hotspot analysis identifies regions with a high number of points and, like
range queries, it is fundamental in location analytics.
For example, businesses need to identify popular regions for advertising, and city agencies have to manage congestion and traffic flow.
We analyze hotspots for the three two-dimensional datasets by generating kernel density estimates (KDEs) for the real and synthetic datasets.
The KDE uses the two-dimensional Gaussian kernel defined over a uniform grid with dimensions of $g \times g$, where $g$ denotes the granularity.  
We use a range of granularities: $g = \{2^6, 2^7, 2^8, 2^9, 2^{10}\}$, and define hotspots to be grid cells in which the density is greater than the 95\textsuperscript{th} percentile.  
For each of the 60 samples, we assess query response similarity between the real and synthetic data using the S\o{}rensen-Dice coefficient (SDC):
\begin{equation}
\textstyle
    SDC = \frac{2\big|\mathcal{H}_{\rreal} \bigcap \mathcal{H}_{\synth}\big|}{\big|\mathcal{H}_{\rreal}\big| +  \big|\mathcal{H}_{\synth}\big|}
    \label{eq:sorensen-dice}
\end{equation}
where $\mathcal{H}$ is the set of hotspots.  

\begin{figure}[t]
\centering
    \begin{subfigure}[b]{\columnwidth}
        \centering
        \includegraphics[height=0.5cm]{Figures Raw/legend2.pdf}
    \end{subfigure}
    \\
    \begin{subfigure}[b]{0.34\columnwidth}
        \centering
        \includegraphics[height=4cm]{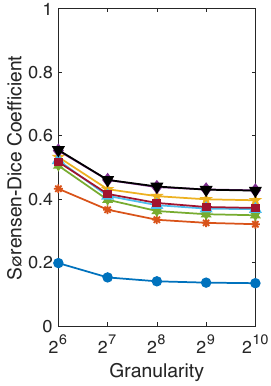}
        \vspace{-0.2cm}
        \caption{Porto}
        \label{fig:porto-hotspot}
    \end{subfigure}
    \hfill
    \begin{subfigure}[b]{0.32\columnwidth}
        \centering
        \includegraphics[height=4cm]{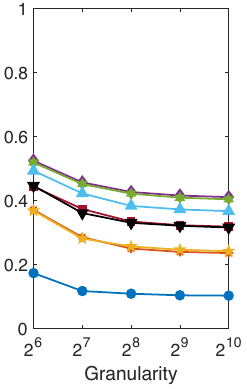}
        \vspace{-0.2cm}
        \caption{Beijing}
        \label{fig:beijing-hotspot}
    \end{subfigure}
   \hfill
    \begin{subfigure}[b]{0.32\columnwidth}
        \centering
        \includegraphics[height=4cm]{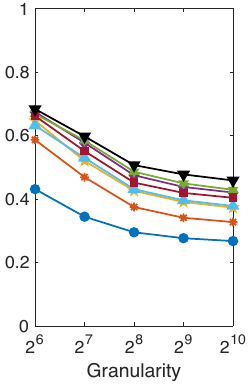}
        \vspace{-0.2cm}
        \caption{New York}
        \label{fig:nyc-hotspot}
    \end{subfigure}
    \vspace{-0.7cm}
    \caption{Variation in SDC as the hotspot granularity varies}
    \label{fig:hotspots}
    \vskip -0.1in
\end{figure}

Figure \ref{fig:hotspots} shows the variation in mean SDC values as the hotspot granularity increases and, once again, we observe similar findings.
Namely, (a) poor performance is observed when $\epsilon = 0.1$, while other $\epsilon$ values are competitive with the non-private GeoPointGAN; (b) private GeoPointGANs sometimes outperform the non-private version; and (c) private GeoPointGANs perform best with a middling $\epsilon$ value, though the exact value depends on the city.

\subsection{Facility Location Queries}
\label{ss:fl}
Facility location is an example (of which there are many) of an end application for our work.
Facility location queries are more complex as they are a combination of range and hotspot queries.
We consider two variants: \textsc{Max-Inf} and \textsc{Min-Dist}.  
In the \textsc{Max-Inf} case, we aim to select the most influential candidate facilities, where influence is commonly defined as the total number of customers that the facilities attract.  
The \textsc{Min-Dist} query instead selects facilities that minimize the total distance between customers and the facilities. 
For both queries, individual location data is needed to accurately model the behaviors of potential customers, which motivates the use of models like GeoPointGAN.

\para{Outline}
Imagine a hot dog salesperson that wishes to locate $k$ outlets throughout Manhattan.  
Since more business could potentially be generated if her outlets were located at the intersections of busy streets, we use the same location set as used for the range queries (as this was a selection of road intersections).
Each $\pi \in \mathcal{P}$ represents a candidate facility, we consider 100 facilities, and assume that there are no existing facilities.  
$\mathcal{F_{\rreal}}$ and $\mathcal{F_{\synth}}$ denote the sets of selected facilities when the real and synthetic data is used, respectively.
To quantify similarity between $\mathcal{F_{\rreal}}$ and $\mathcal{F_{\synth}}$, we use the SDC (Equation \ref{eq:sorensen-dice}), using $\mathcal{F}$ in place of $\mathcal{H}$.
Our evaluation uses $k = \{1, 5, 10, 20, 50, 75\}$.

\para{Results}
Figure \ref{fig:fl} shows the variation in SDC values for the \textsc{Max-Inf} query. 
GeoPointGAN produces synthetic data that answers queries with high accuracy, and the non-private GeoPointGAN performs exceptionally well in most cases, especially in Porto for which it obtains near-optimal results. 
The effect of changing the privacy budget is also noticeable, and we continue to see the phenomenon of privatized versions of GeoPointGAN performing as well as non-private versions.
We conduct the same analysis for the \textsc{Min-Dist} query and similarly strong results are obtained for all values of $\epsilon$ and $k$.  
We omit the corresponding plots due to space limitations.

These strong results demonstrate the practical benefit of our approach and illustrate that facility location can tolerate the noise that is inherent in GAN-based sampling and that is required for label-LDP.
This robustness can also be exploited for other data science tasks, such as nearest neighbor queries and clustering. 

\begin{figure}[t]
\centering
    \begin{subfigure}[b]{\columnwidth}
        \centering
        \includegraphics[height=0.5cm]{Figures Raw/legend2.pdf}
    \end{subfigure}
    \\
    \begin{subfigure}[b]{0.34\columnwidth}
        \centering
        \includegraphics[height=4cm]{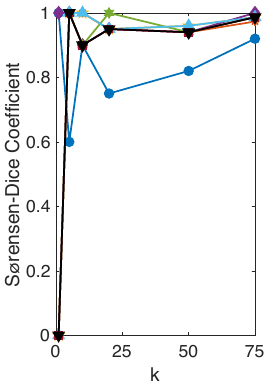}
        \vspace{-0.2cm}
        \caption{Porto}
        \label{fig:porto-fl}
    \end{subfigure}
    \hfill
    \begin{subfigure}[b]{0.32\columnwidth}
        \centering
        \includegraphics[height=4cm]{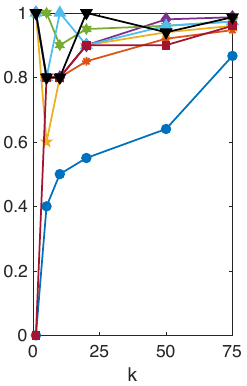}
        \vspace{-0.2cm}
        \caption{Beijing}
        \label{fig:beijing-fl}
    \end{subfigure}
   \hfill
    \begin{subfigure}[b]{0.32\columnwidth}
        \centering
        \includegraphics[height=4cm]{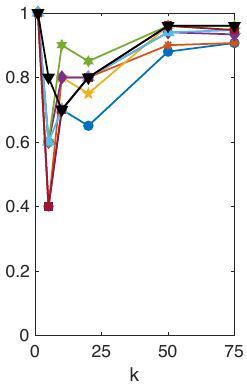}
        \vspace{-0.2cm}
        \caption{New York}
        \label{fig:nyc-fl}
    \end{subfigure}
    \vspace{-0.7cm}
    \caption{Variation in SDC as $\bm{k}$ varies for the \textsc{Max-Inf} query}
    \label{fig:fl}
    \vspace{-0.4cm}
\end{figure}

\section{Final Remarks}
\label{s:conc}
As demonstrated through our experiments, GeoPointGAN is a robust method for generating large synthetic spatial datasets with practical levels of privacy and high utility, both statistically and with respect to several location analytics tasks.
Indeed, in some settings, private GeoPointGANs perform better than non-private versions -- a remarkable and important observation.
This phenomenon is possible due to the design of our privacy mechanism, which exploits the inherent noise in label flipping to harness the regularization effects that can be realized when training GANs.

Beyond these findings, our work provides further insights about the level of privacy provided.
While GeoPointGAN demonstrably fulfills the requirements of label-LDP, the actual level of privacy provided is higher in many practical settings.
As \citet{Malek2021} note, simply removing the sensitive labels of a public dataset and training a model in an unsupervised fashion, complies with label-(L)DP.
We take this further by accounting for situations where, even if we remove the identifier (e.g., taxi ID, 311 caller name), we are still conscious of leaking information based on knowledge regarding the veracity of each point (e.g., if locations can help to identify the caller).
Specifically, we use pseudo-labels (to denote whether points are real or fake) to obfuscate the training data by removing certainty within the model as to what input represents real data.
In this sense, the level of privacy provided is stronger than what is necessarily required with label-LDP, although quantifying this achieved level of privacy is non-trivial.
And, while this level of privacy may not necessarily be strong enough to satisfy LDP in theory, it may satisfy LDP to some extent in practice.
We leave further exploration of these ideas and challenges for future work.

\para{\textbf{Acknowledgments}}
This work is supported in part by the \grantsponsor{EPSRC}{UK Engineering and Physical Sciences Research Council} u under Grant No.~ \grantnum{EPSRC}{EP/L016400/1}.

\balance

\clearpage

\bibliographystyle{ACM-Reference-Format}
\bibliography{97-geopointgan}


\begin{thebibliography}{78}


\ifx \showCODEN    \undefined \def \showCODEN     #1{\unskip}     \fi
\ifx \showDOI      \undefined \def \showDOI       #1{#1}\fi
\ifx \showISBNx    \undefined \def \showISBNx     #1{\unskip}     \fi
\ifx \showISBNxiii \undefined \def \showISBNxiii  #1{\unskip}     \fi
\ifx \showISSN     \undefined \def \showISSN      #1{\unskip}     \fi
\ifx \showLCCN     \undefined \def \showLCCN      #1{\unskip}     \fi
\ifx \shownote     \undefined \def \shownote      #1{#1}          \fi
\ifx \showarticletitle \undefined \def \showarticletitle #1{#1}   \fi
\ifx \showURL      \undefined \def \showURL       {\relax}        \fi
\providecommand\bibfield[2]{#2}
\providecommand\bibinfo[2]{#2}
\providecommand\natexlab[1]{#1}
\providecommand\showeprint[2][]{arXiv:#2}

\bibitem[\protect\citeauthoryear{Achlioptas, Diamanti, Mitliagkas, and
  Guibas}{Achlioptas et~al\mbox{.}}{2018}]%
        {Achlioptas2018}
\bibfield{author}{\bibinfo{person}{Panos Achlioptas}, \bibinfo{person}{Olga
  Diamanti}, \bibinfo{person}{Ioannis Mitliagkas}, {and}
  \bibinfo{person}{Leonidas Guibas}.} \bibinfo{year}{2018}\natexlab{}.
\newblock \showarticletitle{Learning Representations and Generative Models for
  3{D} Point Clouds}. In \bibinfo{booktitle}{\emph{ICML}}.
  \bibinfo{pages}{40--49}.
\newblock
\urldef\tempurl%
\url{http://proceedings.mlr.press/v80/achlioptas18a.html}
\showURL{%
\tempurl}


\bibitem[\protect\citeauthoryear{Acs and Castelluccia}{Acs and
  Castelluccia}{2014}]%
        {Acs2014}
\bibfield{author}{\bibinfo{person}{Gergely Acs} {and} \bibinfo{person}{Claude
  Castelluccia}.} \bibinfo{year}{2014}\natexlab{}.
\newblock \showarticletitle{A Case Study: Privacy Preserving Release of
  Spatio-Temporal Density in Paris}. In \bibinfo{booktitle}{\emph{ACM SIGKDD}}.
  \bibinfo{pages}{1679–1688}.
\newblock
\urldef\tempurl%
\url{https://doi.org/10.1145/2623330.2623361}
\showDOI{\tempurl}


\bibitem[\protect\citeauthoryear{Akbari and Liang}{Akbari and Liang}{2018}]%
        {Akbari2018}
\bibfield{author}{\bibinfo{person}{Mohammad Akbari} {and} \bibinfo{person}{Jie
  Liang}.} \bibinfo{year}{2018}\natexlab{}.
\newblock \showarticletitle{{Semi-Recurrent Cnn-Based Vae-Gan for Sequential
  Data Generation}}. In \bibinfo{booktitle}{\emph{IEEE International Conference
  on Acoustics, Speech and Signal Processing}}.
\newblock
\urldef\tempurl%
\url{https://doi.org/10.1109/ICASSP.2018.8461724}
\showDOI{\tempurl}


\bibitem[\protect\citeauthoryear{Andr{\'e}s, Bordenabe, Chatzikokolakis, and
  Palamidessi}{Andr{\'e}s et~al\mbox{.}}{2013}]%
        {Andres2013}
\bibfield{author}{\bibinfo{person}{Miguel~E. Andr{\'e}s},
  \bibinfo{person}{Nicol\'{a}s~E. Bordenabe}, \bibinfo{person}{Konstantinos
  Chatzikokolakis}, {and} \bibinfo{person}{Catuscia Palamidessi}.}
  \bibinfo{year}{2013}\natexlab{}.
\newblock \showarticletitle{Geo-indistinguishability: differential privacy for
  location-based systems}. In \bibinfo{booktitle}{\emph{ACM SIGSAC}}.
  \bibinfo{pages}{901--914}.
\newblock
\urldef\tempurl%
\url{https://doi.org/10.1145/2508859.2516735}
\showDOI{\tempurl}


\bibitem[\protect\citeauthoryear{Arshad and Beksi}{Arshad and Beksi}{2020}]%
        {Arshad2020}
\bibfield{author}{\bibinfo{person}{Mohammad~Samiul Arshad} {and}
  \bibinfo{person}{William~J. Beksi}.} \bibinfo{year}{2020}\natexlab{}.
\newblock \showarticletitle{{A Progressive Conditional Generative Adversarial
  Network for Generating Dense and Colored 3D Point Clouds}}. In
  \bibinfo{booktitle}{\emph{International Conference on 3D Vision}}.
  \bibinfo{pages}{712--722}.
\newblock
\urldef\tempurl%
\url{https://doi.org/10.1109/3DV50981.2020.00081}
\showDOI{\tempurl}


\bibitem[\protect\citeauthoryear{Augenstein, McMahan, Ramage, Ramaswamy,
  Kairouz, Chen, Mathews, and y~Arcas}{Augenstein et~al\mbox{.}}{2019}]%
        {Augenstein2019}
\bibfield{author}{\bibinfo{person}{Sean Augenstein},
  \bibinfo{person}{H.~Brendan McMahan}, \bibinfo{person}{Daniel Ramage},
  \bibinfo{person}{Swaroop Ramaswamy}, \bibinfo{person}{Peter Kairouz},
  \bibinfo{person}{Mingqing Chen}, \bibinfo{person}{Rajiv Mathews}, {and}
  \bibinfo{person}{Blaise~Ag{\"{u}}era y Arcas}.}
  \bibinfo{year}{2019}\natexlab{}.
\newblock \showarticletitle{Generative Models for Effective {ML} on Private,
  Decentralized Datasets}.
\newblock  (\bibinfo{year}{2019}).
\newblock
\showeprint[arxiv]{1911.06679}


\bibitem[\protect\citeauthoryear{Beaulieu-Jones, Wu, Williams, Lee, Bhavnani,
  Byrd, and Greene}{Beaulieu-Jones et~al\mbox{.}}{2019}]%
        {Beaulieu-Jones2019}
\bibfield{author}{\bibinfo{person}{Brett~K. Beaulieu-Jones},
  \bibinfo{person}{Zhiwei~Steven Wu}, \bibinfo{person}{Chris Williams},
  \bibinfo{person}{Ran Lee}, \bibinfo{person}{Sanjeev~P. Bhavnani},
  \bibinfo{person}{James~Brian Byrd}, {and} \bibinfo{person}{Casey~S. Greene}.}
  \bibinfo{year}{2019}\natexlab{}.
\newblock \showarticletitle{Privacy-Preserving Generative Deep Neural Networks
  Support Clinical Data Sharing}.
\newblock \bibinfo{journal}{\emph{Circulation: Cardiovascular Quality and
  Outcomes}} \bibinfo{volume}{12}, \bibinfo{number}{7} (\bibinfo{year}{2019}),
  \bibinfo{pages}{e005122}.
\newblock
\urldef\tempurl%
\url{https://doi.org/10.1161/CIRCOUTCOMES.118.005122}
\showDOI{\tempurl}


\bibitem[\protect\citeauthoryear{Binnig, Kossmann, Lo, and \"{O}zsu}{Binnig
  et~al\mbox{.}}{2007}]%
        {Binnig2007}
\bibfield{author}{\bibinfo{person}{Carsten Binnig}, \bibinfo{person}{Donald
  Kossmann}, \bibinfo{person}{Eric Lo}, {and} \bibinfo{person}{M.~Tamer
  \"{O}zsu}.} \bibinfo{year}{2007}\natexlab{}.
\newblock \showarticletitle{QAGen: Generating Query-Aware Test Databases}. In
  \bibinfo{booktitle}{\emph{ACM SIGMOD}}. \bibinfo{pages}{341–352}.
\newblock
\urldef\tempurl%
\url{https://doi.org/10.1145/1247480.1247520}
\showDOI{\tempurl}


\bibitem[\protect\citeauthoryear{Bruno and Chaudhuri}{Bruno and
  Chaudhuri}{2005}]%
        {Bruno2005}
\bibfield{author}{\bibinfo{person}{Nicolas Bruno} {and}
  \bibinfo{person}{Surajit Chaudhuri}.} \bibinfo{year}{2005}\natexlab{}.
\newblock \showarticletitle{Flexible Database Generators}. In
  \bibinfo{booktitle}{\emph{PVLDB}}. \bibinfo{publisher}{VLDB Endowment},
  \bibinfo{pages}{1097–1107}.
\newblock


\bibitem[\protect\citeauthoryear{Busa-Fekete, Medina, Syed, and
  Vassilvitskii}{Busa-Fekete et~al\mbox{.}}{2021}]%
        {BusaFekete2021}
\bibfield{author}{\bibinfo{person}{Robert~Istvan Busa-Fekete},
  \bibinfo{person}{Andres~Munoz Medina}, \bibinfo{person}{Umar Syed}, {and}
  \bibinfo{person}{Sergei Vassilvitskii}.} \bibinfo{year}{2021}\natexlab{}.
\newblock \showarticletitle{Population Level Privacy Leakage in Binary
  Classification wtih Label Noise}. In \bibinfo{booktitle}{\emph{NeurIPS}}.
\newblock
\urldef\tempurl%
\url{https://openreview.net/forum?id=Gf2EuAB9Xj}
\showURL{%
\tempurl}


\bibitem[\protect\citeauthoryear{Caccia, Hoof, Courville, and Pineau}{Caccia
  et~al\mbox{.}}{2019}]%
        {Caccia2019}
\bibfield{author}{\bibinfo{person}{Lucas Caccia}, \bibinfo{person}{Herke~Van
  Hoof}, \bibinfo{person}{Aaron Courville}, {and} \bibinfo{person}{Joelle
  Pineau}.} \bibinfo{year}{2019}\natexlab{}.
\newblock \showarticletitle{{Deep Generative Modeling of LiDAR Data}}. In
  \bibinfo{booktitle}{\emph{IEEE International Conference on Intelligent Robots
  and Systems}}.
\newblock
\urldef\tempurl%
\url{https://doi.org/10.1109/IROS40897.2019.8968535}
\showDOI{\tempurl}


\bibitem[\protect\citeauthoryear{Chang, Funkhouser, Guibas, Hanrahan, Huang,
  Li, Savarese, Savva, Song, Su, Xiao, Yi, and Yu}{Chang et~al\mbox{.}}{2015}]%
        {Chang2015}
\bibfield{author}{\bibinfo{person}{Angel~X. Chang}, \bibinfo{person}{Thomas
  Funkhouser}, \bibinfo{person}{Leonidas Guibas}, \bibinfo{person}{Pat
  Hanrahan}, \bibinfo{person}{Qixing Huang}, \bibinfo{person}{Zimo Li},
  \bibinfo{person}{Silvio Savarese}, \bibinfo{person}{Manolis Savva},
  \bibinfo{person}{Shuran Song}, \bibinfo{person}{Hao Su},
  \bibinfo{person}{Jianxiong Xiao}, \bibinfo{person}{Li Yi}, {and}
  \bibinfo{person}{Fisher Yu}.} \bibinfo{year}{2015}\natexlab{}.
\newblock \showarticletitle{{ShapeNet: An Information-Rich 3D Model
  Repository}}.
\newblock  (\bibinfo{year}{2015}).
\newblock
\urldef\tempurl%
\url{http://arxiv.org/abs/1512.03012}
\showURL{%
\tempurl}


\bibitem[\protect\citeauthoryear{Chaudhuri and Hsu}{Chaudhuri and Hsu}{2011}]%
        {Chaudhuri2011}
\bibfield{author}{\bibinfo{person}{Kamalika Chaudhuri} {and}
  \bibinfo{person}{Daniel Hsu}.} \bibinfo{year}{2011}\natexlab{}.
\newblock \showarticletitle{Sample Complexity Bounds for Differentially Private
  Learning}. In \bibinfo{booktitle}{\emph{Conference on Learning Theory}},
  Vol.~\bibinfo{volume}{19}. \bibinfo{publisher}{PMLR},
  \bibinfo{pages}{155--186}.
\newblock
\urldef\tempurl%
\url{https://proceedings.mlr.press/v19/chaudhuri11a.html}
\showURL{%
\tempurl}


\bibitem[\protect\citeauthoryear{Chen, Dai, Tao, Shen, Gan, Zhang, Zhang, and
  Carin}{Chen et~al\mbox{.}}{2018}]%
        {Chen2018}
\bibfield{author}{\bibinfo{person}{Liqun Chen}, \bibinfo{person}{Shuyang Dai},
  \bibinfo{person}{Chenyang Tao}, \bibinfo{person}{Dinghan Shen},
  \bibinfo{person}{Zhe Gan}, \bibinfo{person}{Haichao Zhang},
  \bibinfo{person}{Yizhe Zhang}, {and} \bibinfo{person}{Lawrence Carin}.}
  \bibinfo{year}{2018}\natexlab{}.
\newblock \showarticletitle{{Adversarial text generation via feature-mover's
  distance}}. In \bibinfo{booktitle}{\emph{NeurIPS}}.
\newblock


\bibitem[\protect\citeauthoryear{Chen, Fung, Desai, and Sossou}{Chen
  et~al\mbox{.}}{2012}]%
        {Chen2012}
\bibfield{author}{\bibinfo{person}{Rui Chen}, \bibinfo{person}{Benjamin~C.M.
  Fung}, \bibinfo{person}{Bipin~C. Desai}, {and} \bibinfo{person}{N\'{e}riah~M.
  Sossou}.} \bibinfo{year}{2012}\natexlab{}.
\newblock \bibinfo{booktitle}{\emph{Differentially Private Transit Data
  Publication: A Case Study on the Montreal Transportation System}}.
\newblock \bibinfo{pages}{213–221}.
\newblock
\urldef\tempurl%
\url{https://doi.org/10.1145/2339530.2339564}
\showDOI{\tempurl}


\bibitem[\protect\citeauthoryear{Chen, Li, Qin, Kasiviswanathan, and Jin}{Chen
  et~al\mbox{.}}{2016}]%
        {Chen2016}
\bibfield{author}{\bibinfo{person}{Rui Chen}, \bibinfo{person}{Haoran Li},
  \bibinfo{person}{A.~K. Qin}, \bibinfo{person}{Shiva~Prasad Kasiviswanathan},
  {and} \bibinfo{person}{Hongxia Jin}.} \bibinfo{year}{2016}\natexlab{}.
\newblock \showarticletitle{Private spatial data aggregation in the local
  setting}. In \bibinfo{booktitle}{\emph{IEEE ICDE}}.
\newblock
\urldef\tempurl%
\url{https://doi.org/10.1109/ICDE.2016.7498248}
\showDOI{\tempurl}


\bibitem[\protect\citeauthoryear{Cormode, Procopiuc, Srivastava, Shen, and
  Yu}{Cormode et~al\mbox{.}}{2012}]%
        {Cormode2012}
\bibfield{author}{\bibinfo{person}{Graham Cormode}, \bibinfo{person}{Cecilia
  Procopiuc}, \bibinfo{person}{Divesh Srivastava}, \bibinfo{person}{Entong
  Shen}, {and} \bibinfo{person}{Ting Yu}.} \bibinfo{year}{2012}\natexlab{}.
\newblock \showarticletitle{Differentially Private Spatial Decompositions}. In
  \bibinfo{booktitle}{\emph{IEEE ICDE}}.
\newblock
\urldef\tempurl%
\url{https://doi.org/10.1109/ICDE.2012.16}
\showDOI{\tempurl}


\bibitem[\protect\citeauthoryear{Cunningham, Cormode, and
  Ferhatosmanoglu}{Cunningham et~al\mbox{.}}{2021a}]%
        {Cunningham2021}
\bibfield{author}{\bibinfo{person}{Teddy Cunningham}, \bibinfo{person}{Graham
  Cormode}, {and} \bibinfo{person}{Hakan Ferhatosmanoglu}.}
  \bibinfo{year}{2021}\natexlab{a}.
\newblock \showarticletitle{Privacy-Preserving Synthetic Location Data in the
  Real World}. In \bibinfo{booktitle}{\emph{SSTD}}.
\newblock
\urldef\tempurl%
\url{https://doi.org/10.1145/3469830.3470893}
\showDOI{\tempurl}


\bibitem[\protect\citeauthoryear{Cunningham, Cormode, Ferhatosmanoglu, and
  Srivastava}{Cunningham et~al\mbox{.}}{2021b}]%
        {Cunningham2021a}
\bibfield{author}{\bibinfo{person}{Teddy Cunningham}, \bibinfo{person}{Graham
  Cormode}, \bibinfo{person}{Hakan Ferhatosmanoglu}, {and}
  \bibinfo{person}{Divesh Srivastava}.} \bibinfo{year}{2021}\natexlab{b}.
\newblock \showarticletitle{Real-World Trajectory Sharing with Local
  Differential Privacy}.
\newblock \bibinfo{journal}{\emph{PVLDB}} \bibinfo{volume}{14},
  \bibinfo{number}{11} (\bibinfo{year}{2021}).
\newblock
\urldef\tempurl%
\url{https://doi.org/10.14778/3476249.3476280}
\showDOI{\tempurl}


\bibitem[\protect\citeauthoryear{Del~Rey, Xiong, Liu, Li, Cai, and Niu}{Del~Rey
  et~al\mbox{.}}{2020}]%
        {DelRey2020}
\bibfield{author}{\bibinfo{person}{Angel~M. Del~Rey}, \bibinfo{person}{Xingxing
  Xiong}, \bibinfo{person}{Shubo Liu}, \bibinfo{person}{Dan Li},
  \bibinfo{person}{Zhaohui Cai}, {and} \bibinfo{person}{Xiaoguang Niu}.}
  \bibinfo{year}{2020}\natexlab{}.
\newblock \showarticletitle{A Comprehensive Survey on Local Differential
  Privacy}.
\newblock \bibinfo{journal}{\emph{Security and Communication Networks}}
  (\bibinfo{year}{2020}).
\newblock
\urldef\tempurl%
\url{https://doi.org/10.1155/2020/8829523}
\showDOI{\tempurl}


\bibitem[\protect\citeauthoryear{Dizaji, Wang, and Huang}{Dizaji
  et~al\mbox{.}}{2018}]%
        {Dizaji2018}
\bibfield{author}{\bibinfo{person}{Kamran~Ghasedi Dizaji},
  \bibinfo{person}{Xiaoqian Wang}, {and} \bibinfo{person}{Heng Huang}.}
  \bibinfo{year}{2018}\natexlab{}.
\newblock \showarticletitle{{Semi-supervised generative adversarial network for
  gene expression inference}}. In \bibinfo{booktitle}{\emph{ACM SIGKDD}}.
\newblock
\urldef\tempurl%
\url{https://doi.org/10.1145/3219819.3220114}
\showDOI{\tempurl}


\bibitem[\protect\citeauthoryear{{Duchi}, {Jordan}, and {Wainwright}}{{Duchi}
  et~al\mbox{.}}{2013}]%
        {Duchi2013}
\bibfield{author}{\bibinfo{person}{John~C. {Duchi}},
  \bibinfo{person}{Michael~I. {Jordan}}, {and} \bibinfo{person}{Martin~J.
  {Wainwright}}.} \bibinfo{year}{2013}\natexlab{}.
\newblock \showarticletitle{Local Privacy and Statistical Minimax Rates}. In
  \bibinfo{booktitle}{\emph{IEEE FOCS}}. \bibinfo{pages}{429--438}.
\newblock
\urldef\tempurl%
\url{https://doi.org/10.1109/FOCS.2013.53}
\showDOI{\tempurl}


\bibitem[\protect\citeauthoryear{Dwork}{Dwork}{2006}]%
        {Dwork2006}
\bibfield{author}{\bibinfo{person}{Cynthia Dwork}.}
  \bibinfo{year}{2006}\natexlab{}.
\newblock \showarticletitle{Differential Privacy}. In
  \bibinfo{booktitle}{\emph{Automata, Languages and Programming}}.
  \bibinfo{publisher}{Springer Berlin Heidelberg}, \bibinfo{address}{Berlin,
  Heidelberg}, \bibinfo{pages}{1--12}.
\newblock


\bibitem[\protect\citeauthoryear{Esfandiari, Mirrokni, Syed, and
  Vassilvitskii}{Esfandiari et~al\mbox{.}}{2021}]%
        {Esfandiari2021}
\bibfield{author}{\bibinfo{person}{Hossein Esfandiari}, \bibinfo{person}{Vahab
  Mirrokni}, \bibinfo{person}{Umar Syed}, {and} \bibinfo{person}{Sergei
  Vassilvitskii}.} \bibinfo{year}{2021}\natexlab{}.
\newblock \bibinfo{title}{Label differential privacy via clustering}.
\newblock
\newblock
\showeprint[arxiv]{2110.02159}


\bibitem[\protect\citeauthoryear{Frigerio, de~Oliveira, Gomez, and
  Duverger}{Frigerio et~al\mbox{.}}{2019}]%
        {Frigerio2019}
\bibfield{author}{\bibinfo{person}{Lorenzo Frigerio},
  \bibinfo{person}{Anderson~Santana de Oliveira}, \bibinfo{person}{Laurent
  Gomez}, {and} \bibinfo{person}{Patrick Duverger}.}
  \bibinfo{year}{2019}\natexlab{}.
\newblock \showarticletitle{Differentially Private Generative Adversarial
  Networks for Time Series, Continuous, and Discrete Open Data}. In
  \bibinfo{booktitle}{\emph{ICT Systems Security and Privacy Protection}}.
  \bibinfo{pages}{151--164}.
\newblock


\bibitem[\protect\citeauthoryear{Gal, Bermano, Zhang, and Cohen-Or}{Gal
  et~al\mbox{.}}{2020}]%
        {Gal2020}
\bibfield{author}{\bibinfo{person}{Rinon Gal}, \bibinfo{person}{Amit Bermano},
  \bibinfo{person}{Hao Zhang}, {and} \bibinfo{person}{Daniel Cohen-Or}.}
  \bibinfo{year}{2020}\natexlab{}.
\newblock \bibinfo{title}{{MRGAN: Multi-Rooted 3D Shape Generation with
  Unsupervised Part Disentanglement}}.
\newblock
\newblock
\showISSN{23318422}
\showeprint[arxiv]{2007.12944}


\bibitem[\protect\citeauthoryear{Gatrell, Bailey, Diggle, and
  Rowlingson}{Gatrell et~al\mbox{.}}{1996}]%
        {Gatrell1996}
\bibfield{author}{\bibinfo{person}{Anthony~C. Gatrell},
  \bibinfo{person}{Trevor~C. Bailey}, \bibinfo{person}{Peter~J. Diggle}, {and}
  \bibinfo{person}{Barry~S. Rowlingson}.} \bibinfo{year}{1996}\natexlab{}.
\newblock \showarticletitle{{Spatial Point Pattern Analysis and Its Application
  in Geographical Epidemiology}}.
\newblock \bibinfo{journal}{\emph{Transactions of the Institute of British
  Geographers}} (\bibinfo{year}{1996}).
\newblock
\showISSN{00202754}
\urldef\tempurl%
\url{https://doi.org/10.2307/622936}
\showDOI{\tempurl}


\bibitem[\protect\citeauthoryear{Ge, Mohapatra, He, and Ilyas}{Ge
  et~al\mbox{.}}{2021}]%
        {Ge2021}
\bibfield{author}{\bibinfo{person}{Chang Ge}, \bibinfo{person}{Shubhankar
  Mohapatra}, \bibinfo{person}{Xi He}, {and} \bibinfo{person}{Ihab~F. Ilyas}.}
  \bibinfo{year}{2021}\natexlab{}.
\newblock \bibinfo{title}{Kamino: Constraint-Aware Differentially Private Data
  Synthesis}.
\newblock
\newblock
\showeprint[arxiv]{2012.15713}~[cs.DB]


\bibitem[\protect\citeauthoryear{Geolink}{Geolink}{2015}]%
        {Porto2015}
\bibfield{author}{\bibinfo{person}{Geolink}.} \bibinfo{year}{2015}\natexlab{}.
\newblock \bibinfo{title}{{Taxi Service Trajectory Prediction Challenge @
  ECML/PKDD 15 -- Dataset}}.
\newblock
\newblock
\urldef\tempurl%
\url{http://www.geolink.pt/ecmlpkdd2015-challenge/dataset.html}
\showURL{%
Retrieved August 15, 2019 from \tempurl}


\bibitem[\protect\citeauthoryear{Ghane, Kulik, and Ramamohanarao}{Ghane
  et~al\mbox{.}}{2018}]%
        {Ghane2018}
\bibfield{author}{\bibinfo{person}{Soheila Ghane}, \bibinfo{person}{Lars
  Kulik}, {and} \bibinfo{person}{Kotagiri Ramamohanarao}.}
  \bibinfo{year}{2018}\natexlab{}.
\newblock \showarticletitle{Publishing Spatial Histograms under Differential
  Privacy}. In \bibinfo{booktitle}{\emph{SSDBM}}. Article
  \bibinfo{articleno}{14}.
\newblock
\urldef\tempurl%
\url{https://doi.org/10.1145/3221269.3223039}
\showDOI{\tempurl}


\bibitem[\protect\citeauthoryear{Ghazal, Rabl, Hu, Raab, Poess, Crolotte, and
  Jacobsen}{Ghazal et~al\mbox{.}}{2013}]%
        {Ghazal2013}
\bibfield{author}{\bibinfo{person}{Ahmad Ghazal}, \bibinfo{person}{Tilmann
  Rabl}, \bibinfo{person}{Minqing Hu}, \bibinfo{person}{Francois Raab},
  \bibinfo{person}{Meikel Poess}, \bibinfo{person}{Alain Crolotte}, {and}
  \bibinfo{person}{Hans-Arno Jacobsen}.} \bibinfo{year}{2013}\natexlab{}.
\newblock \showarticletitle{BigBench: Towards an Industry Standard Benchmark
  for Big Data Analytics}. In \bibinfo{booktitle}{\emph{ACM SIGMOD}}.
  \bibinfo{pages}{1197–1208}.
\newblock
\urldef\tempurl%
\url{https://doi.org/10.1145/2463676.2463712}
\showDOI{\tempurl}


\bibitem[\protect\citeauthoryear{Ghazi, Golowich, Kumar, Manurangsi, and
  Zhang}{Ghazi et~al\mbox{.}}{2021}]%
        {Ghazi2021}
\bibfield{author}{\bibinfo{person}{Badih Ghazi}, \bibinfo{person}{Noah
  Golowich}, \bibinfo{person}{Ravi Kumar}, \bibinfo{person}{Pasin Manurangsi},
  {and} \bibinfo{person}{Chiyuan Zhang}.} \bibinfo{year}{2021}\natexlab{}.
\newblock \showarticletitle{Deep Learning with Label Differential Privacy}. In
  \bibinfo{booktitle}{\emph{NeurIPS}},
  \bibfield{editor}{\bibinfo{person}{A.~Beygelzimer},
  \bibinfo{person}{Y.~Dauphin}, \bibinfo{person}{P.~Liang}, {and}
  \bibinfo{person}{J.~Wortman Vaughan}} (Eds.).
\newblock
\urldef\tempurl%
\url{https://openreview.net/forum?id=RYcgfqmAOHh}
\showURL{%
\tempurl}


\bibitem[\protect\citeauthoryear{Goodfellow, Pouget-Abadie, Mirza, Xu,
  Warde-Farley, Ozair, Courville, and Bengio}{Goodfellow et~al\mbox{.}}{2014}]%
        {Goodfellow2014}
\bibfield{author}{\bibinfo{person}{Ian Goodfellow}, \bibinfo{person}{Jean
  Pouget-Abadie}, \bibinfo{person}{Mehdi Mirza}, \bibinfo{person}{Bing Xu},
  \bibinfo{person}{David Warde-Farley}, \bibinfo{person}{Sherjil Ozair},
  \bibinfo{person}{Aaron Courville}, {and} \bibinfo{person}{Yoshua Bengio}.}
  \bibinfo{year}{2014}\natexlab{}.
\newblock \showarticletitle{{Generative Adversarial Nets}}. In
  \bibinfo{booktitle}{\emph{NeurIPS}}. \bibinfo{pages}{2672--2680}.
\newblock
\urldef\tempurl%
\url{http://papers.nips.cc/paper/5423-generative-adversarial-nets}
\showURL{%
\tempurl}


\bibitem[\protect\citeauthoryear{Gu, Zhou, Zhang, Shan, Zhou, and Winslett}{Gu
  et~al\mbox{.}}{2015}]%
        {Gu2015}
\bibfield{author}{\bibinfo{person}{Ling Gu}, \bibinfo{person}{Minqi Zhou},
  \bibinfo{person}{Zhenjie Zhang}, \bibinfo{person}{Ming-Chien Shan},
  \bibinfo{person}{Aoying Zhou}, {and} \bibinfo{person}{Marianne Winslett}.}
  \bibinfo{year}{2015}\natexlab{}.
\newblock \showarticletitle{Chronos: An elastic parallel framework for stream
  benchmark generation and simulation}. In \bibinfo{booktitle}{\emph{IEEE
  ICDE}}. \bibinfo{pages}{101--112}.
\newblock
\urldef\tempurl%
\url{https://doi.org/10.1109/ICDE.2015.7113276}
\showDOI{\tempurl}


\bibitem[\protect\citeauthoryear{Gursoy, Liu, Truex, and Yu}{Gursoy
  et~al\mbox{.}}{2018}]%
        {Gursoy2018}
\bibfield{author}{\bibinfo{person}{Mehmet~Emre Gursoy}, \bibinfo{person}{Ling
  Liu}, \bibinfo{person}{Stacey Truex}, {and} \bibinfo{person}{Lei Yu}.}
  \bibinfo{year}{2018}\natexlab{}.
\newblock \showarticletitle{Differentially private and utility preserving
  publication of trajectory data}.
\newblock \bibinfo{journal}{\emph{IEEE Transactions on Mobile Computing}}
  \bibinfo{volume}{18}, \bibinfo{number}{10} (\bibinfo{year}{2018}),
  \bibinfo{pages}{2315--2329}.
\newblock


\bibitem[\protect\citeauthoryear{Gursoy, Rajasekar, and Liu}{Gursoy
  et~al\mbox{.}}{2020}]%
        {Gursoy2020}
\bibfield{author}{\bibinfo{person}{Mehmet~Emre Gursoy},
  \bibinfo{person}{Vivekanand Rajasekar}, {and} \bibinfo{person}{Ling Liu}.}
  \bibinfo{year}{2020}\natexlab{}.
\newblock \bibinfo{title}{Utility-Optimized Synthesis of Differentially Private
  Location Traces}.
\newblock
\newblock
\showeprint[arxiv]{2009.06505}


\bibitem[\protect\citeauthoryear{He, Cormode, Machanavajjhala, Procopiuc, and
  Srivastava}{He et~al\mbox{.}}{2015}]%
        {He2015}
\bibfield{author}{\bibinfo{person}{Xi He}, \bibinfo{person}{Graham Cormode},
  \bibinfo{person}{Ashwin Machanavajjhala}, \bibinfo{person}{Cecilia~M
  Procopiuc}, {and} \bibinfo{person}{Divesh Srivastava}.}
  \bibinfo{year}{2015}\natexlab{}.
\newblock \showarticletitle{DPT: differentially private trajectory synthesis
  using hierarchical reference systems}.
\newblock \bibinfo{journal}{\emph{PVLDB}} \bibinfo{volume}{8},
  \bibinfo{number}{11} (\bibinfo{year}{2015}), \bibinfo{pages}{1154--1165}.
\newblock


\bibitem[\protect\citeauthoryear{Huang, McKenna, Bissias, Miklau, Hay, and
  Machanavajjhala}{Huang et~al\mbox{.}}{2019}]%
        {Huang2019}
\bibfield{author}{\bibinfo{person}{Zhiqi Huang}, \bibinfo{person}{Ryan
  McKenna}, \bibinfo{person}{George Bissias}, \bibinfo{person}{Gerome Miklau},
  \bibinfo{person}{Michael Hay}, {and} \bibinfo{person}{Ashwin
  Machanavajjhala}.} \bibinfo{year}{2019}\natexlab{}.
\newblock \showarticletitle{PSynDB: Accurate and Accessible Private Data
  Generation}.
\newblock \bibinfo{journal}{\emph{PVLDB}} \bibinfo{volume}{12},
  \bibinfo{number}{12} (\bibinfo{year}{2019}), \bibinfo{pages}{1918–1921}.
\newblock
\urldef\tempurl%
\url{https://doi.org/10.14778/3352063.3352099}
\showDOI{\tempurl}


\bibitem[\protect\citeauthoryear{Jaderberg, Simonyan, Zisserman, and
  Kavukcuoglu}{Jaderberg et~al\mbox{.}}{2015}]%
        {Jaderberg2015}
\bibfield{author}{\bibinfo{person}{Max Jaderberg}, \bibinfo{person}{Karen
  Simonyan}, \bibinfo{person}{Andrew Zisserman}, {and} \bibinfo{person}{Koray
  Kavukcuoglu}.} \bibinfo{year}{2015}\natexlab{}.
\newblock \showarticletitle{{Spatial transformer networks}}. In
  \bibinfo{booktitle}{\emph{NeurIPS}}.
\newblock


\bibitem[\protect\citeauthoryear{Jiang, Dai, Wu, and Loy}{Jiang
  et~al\mbox{.}}{2021}]%
        {Jiang2021}
\bibfield{author}{\bibinfo{person}{Liming Jiang}, \bibinfo{person}{Bo Dai},
  \bibinfo{person}{Wayne Wu}, {and} \bibinfo{person}{Chen~Change Loy}.}
  \bibinfo{year}{2021}\natexlab{}.
\newblock \showarticletitle{Deceive D: Adaptive Pseudo Augmentation for GAN
  Training with Limited Data}. In \bibinfo{booktitle}{\emph{NeurIPS}},
  Vol.~\bibinfo{volume}{34}. \bibinfo{pages}{21655--21667}.
\newblock
\urldef\tempurl%
\url{https://proceedings.neurips.cc/paper/2021/file/b534ba68236ba543ae44b22bd110a1d6-Paper.pdf}
\showURL{%
\tempurl}


\bibitem[\protect\citeauthoryear{Kaul, Yang, and Jensen}{Kaul
  et~al\mbox{.}}{2013}]%
        {Kaul2013}
\bibfield{author}{\bibinfo{person}{Manohar Kaul}, \bibinfo{person}{Bin Yang},
  {and} \bibinfo{person}{Christian~S. Jensen}.}
  \bibinfo{year}{2013}\natexlab{}.
\newblock \showarticletitle{{Building accurate 3D spatial networks to enable
  next generation intelligent transportation systems}}. In
  \bibinfo{booktitle}{\emph{IEEE Mobile Data Management}}.
\newblock
\urldef\tempurl%
\url{https://doi.org/10.1109/MDM.2013.24}
\showDOI{\tempurl}


\bibitem[\protect\citeauthoryear{Klemmer, Koshiyama, and Flennerhag}{Klemmer
  et~al\mbox{.}}{2019}]%
        {Klemmer2019a}
\bibfield{author}{\bibinfo{person}{Konstantin Klemmer},
  \bibinfo{person}{Adriano Koshiyama}, {and} \bibinfo{person}{Sebastian
  Flennerhag}.} \bibinfo{year}{2019}\natexlab{}.
\newblock \showarticletitle{{Augmenting correlation structures in spatial data
  using deep generative models}}.
\newblock  (\bibinfo{date}{may} \bibinfo{year}{2019}).
\newblock
\showeprint[arxiv]{1905.09796}
\urldef\tempurl%
\url{http://arxiv.org/abs/1905.09796}
\showURL{%
\tempurl}


\bibitem[\protect\citeauthoryear{Klemmer and Neill}{Klemmer and Neill}{2021}]%
        {Klemmer2021a}
\bibfield{author}{\bibinfo{person}{Konstantin Klemmer} {and}
  \bibinfo{person}{Daniel~B. Neill}.} \bibinfo{year}{2021}\natexlab{}.
\newblock \showarticletitle{{Auxiliary-task learning for geographic data with
  autoregressive embeddings}}. In \bibinfo{booktitle}{\emph{SIGSPATIAL:
  Proceedings of the ACM International Symposium on Advances in Geographic
  Information Systems}}.
\newblock


\bibitem[\protect\citeauthoryear{Klemmer, Xu, Acciaio, and Neill}{Klemmer
  et~al\mbox{.}}{2022}]%
        {Klemmer2021d}
\bibfield{author}{\bibinfo{person}{Konstantin Klemmer},
  \bibinfo{person}{Tianlin Xu}, \bibinfo{person}{Beatrice Acciaio}, {and}
  \bibinfo{person}{Daniel~B. Neill}.} \bibinfo{year}{2022}\natexlab{}.
\newblock \showarticletitle{{SPATE-GAN: Improved Generative Modeling of Dynamic
  Spatio-Temporal Patterns with an Autoregressive Embedding Loss}}. In
  \bibinfo{booktitle}{\emph{AAAI 2022 - 36th AAAI Conference on Artificial
  Intelligence}}.
\newblock
\showeprint[arxiv]{2109.15044v1}


\bibitem[\protect\citeauthoryear{Li, Zaheer, Zhang, Poczos, and
  Salakhutdinov}{Li et~al\mbox{.}}{2018}]%
        {Li2018}
\bibfield{author}{\bibinfo{person}{Chun-Liang Li}, \bibinfo{person}{Manzil
  Zaheer}, \bibinfo{person}{Yang Zhang}, \bibinfo{person}{Barnabas Poczos},
  {and} \bibinfo{person}{Ruslan Salakhutdinov}.}
  \bibinfo{year}{2018}\natexlab{}.
\newblock \showarticletitle{{Point Cloud GAN}}.
\newblock  (\bibinfo{year}{2018}).
\newblock
\showeprint[arxiv]{1810.05795}


\bibitem[\protect\citeauthoryear{Li, Li, Fu, Cohen-Or, and Heng}{Li
  et~al\mbox{.}}{2019}]%
        {Li2019a}
\bibfield{author}{\bibinfo{person}{Ruihui Li}, \bibinfo{person}{Xianzhi Li},
  \bibinfo{person}{Chi~Wing Fu}, \bibinfo{person}{Daniel Cohen-Or}, {and}
  \bibinfo{person}{Pheng~Ann Heng}.} \bibinfo{year}{2019}\natexlab{}.
\newblock \showarticletitle{{PU-GAN: A point cloud upsampling adversarial
  network}}. In \bibinfo{booktitle}{\emph{IEEE ICCV}}.
\newblock
\urldef\tempurl%
\url{https://doi.org/10.1109/ICCV.2019.00730}
\showDOI{\tempurl}


\bibitem[\protect\citeauthoryear{Loshchilov and Hutter}{Loshchilov and
  Hutter}{2019}]%
        {Loshchilov2019}
\bibfield{author}{\bibinfo{person}{Ilya Loshchilov} {and}
  \bibinfo{person}{Frank Hutter}.} \bibinfo{year}{2019}\natexlab{}.
\newblock \showarticletitle{{Decoupled weight decay regularization}}. In
  \bibinfo{booktitle}{\emph{ICLR}}.
\newblock
\showeprint[arxiv]{1711.05101}


\bibitem[\protect\citeauthoryear{{Machanavajjhala}, {Kifer}, {Abowd}, {Gehrke},
  and {Vilhuber}}{{Machanavajjhala} et~al\mbox{.}}{2008}]%
        {Machanavajjhala2008}
\bibfield{author}{\bibinfo{person}{Ashwin {Machanavajjhala}},
  \bibinfo{person}{Daniel {Kifer}}, \bibinfo{person}{John {Abowd}},
  \bibinfo{person}{Johannes {Gehrke}}, {and} \bibinfo{person}{Lars
  {Vilhuber}}.} \bibinfo{year}{2008}\natexlab{}.
\newblock \showarticletitle{Privacy: Theory meets Practice on the Map}. In
  \bibinfo{booktitle}{\emph{IEEE ICDE}}. \bibinfo{pages}{277--286}.
\newblock
\urldef\tempurl%
\url{https://doi.org/10.1109/ICDE.2008.4497436}
\showDOI{\tempurl}


\bibitem[\protect\citeauthoryear{Malek, Mironov, Prasad, Shilov, and
  Tramèr}{Malek et~al\mbox{.}}{2021}]%
        {Malek2021}
\bibfield{author}{\bibinfo{person}{Mani Malek}, \bibinfo{person}{Ilya Mironov},
  \bibinfo{person}{Karthik Prasad}, \bibinfo{person}{Igor Shilov}, {and}
  \bibinfo{person}{Florian Tramèr}.} \bibinfo{year}{2021}\natexlab{}.
\newblock \bibinfo{title}{Antipodes of Label Differential Privacy: PATE and
  ALIBI}.
\newblock
\newblock
\showeprint[arxiv]{2106.03408}


\bibitem[\protect\citeauthoryear{Mohammed, Chen, Fung, and Yu}{Mohammed
  et~al\mbox{.}}{2011}]%
        {Mohammed2011}
\bibfield{author}{\bibinfo{person}{Noman Mohammed}, \bibinfo{person}{Rui Chen},
  \bibinfo{person}{Benjamin~C.M. Fung}, {and} \bibinfo{person}{Philip~S. Yu}.}
  \bibinfo{year}{2011}\natexlab{}.
\newblock \showarticletitle{Differentially Private Data Release for Data
  Mining}. In \bibinfo{booktitle}{\emph{ACM SIGKDD}}.
  \bibinfo{pages}{493–501}.
\newblock
\urldef\tempurl%
\url{https://doi.org/10.1145/2020408.2020487}
\showDOI{\tempurl}


\bibitem[\protect\citeauthoryear{{New York City Open Data}}{{New York City Open
  Data}}{2020}]%
        {NYC2020}
\bibfield{author}{\bibinfo{person}{{New York City Open Data}}.}
  \bibinfo{year}{2020}\natexlab{}.
\newblock \bibinfo{title}{311 Service Requests from 2010 to Present}.
\newblock
\newblock
\urldef\tempurl%
\url{https://data.cityofnewyork.us/browse?q=311}
\showURL{%
Retrieved January 23, 2020 from \tempurl}


\bibitem[\protect\citeauthoryear{Nguyen, Mummadi, Ngo, Nguyen, Beggel, and
  Brox}{Nguyen et~al\mbox{.}}{2019}]%
        {Nguyen2019}
\bibfield{author}{\bibinfo{person}{Duc~Tam Nguyen},
  \bibinfo{person}{Chaithanya~Kumar Mummadi}, \bibinfo{person}{Thi Phuong~Nhung
  Ngo}, \bibinfo{person}{Thi Hoai~Phuong Nguyen}, \bibinfo{person}{Laura
  Beggel}, {and} \bibinfo{person}{Thomas Brox}.}
  \bibinfo{year}{2019}\natexlab{}.
\newblock \bibinfo{title}{{Self: Learning to filter noisy labels with
  self-ensembling}}.
\newblock
\newblock
\showISSN{23318422}
\showeprint[arxiv]{1910.01842}


\bibitem[\protect\citeauthoryear{Qi, Su, Mo, and Guibas}{Qi
  et~al\mbox{.}}{2017}]%
        {Qi2017}
\bibfield{author}{\bibinfo{person}{Charles~R. Qi}, \bibinfo{person}{Hao Su},
  \bibinfo{person}{Kaichun Mo}, {and} \bibinfo{person}{Leonidas~J. Guibas}.}
  \bibinfo{year}{2017}\natexlab{}.
\newblock \showarticletitle{{PointNet: Deep learning on point sets for 3D
  classification and segmentation}}. In \bibinfo{booktitle}{\emph{IEEE CVPR}}.
\newblock
\urldef\tempurl%
\url{https://doi.org/10.1109/CVPR.2017.16}
\showDOI{\tempurl}


\bibitem[\protect\citeauthoryear{Qu, Yu, Zhou, and Tian}{Qu
  et~al\mbox{.}}{2020}]%
        {Qu2020}
\bibfield{author}{\bibinfo{person}{Youyang Qu}, \bibinfo{person}{Shui Yu},
  \bibinfo{person}{Wanlei Zhou}, {and} \bibinfo{person}{Yonghong Tian}.}
  \bibinfo{year}{2020}\natexlab{}.
\newblock \showarticletitle{{GAN-Driven Personalized Spatial-Temporal Private
  Data Sharing in Cyber-Physical Social Systems}}.
\newblock \bibinfo{journal}{\emph{IEEE Transactions on Network Science and
  Engineering}} \bibinfo{volume}{7}, \bibinfo{number}{4} (\bibinfo{date}{oct}
  \bibinfo{year}{2020}), \bibinfo{pages}{2576--2586}.
\newblock
\showISSN{23274697}
\urldef\tempurl%
\url{https://doi.org/10.1109/TNSE.2020.3001061}
\showDOI{\tempurl}


\bibitem[\protect\citeauthoryear{Salimans, Goodfellow, Zaremba, Cheung,
  Radford, Chen, and Chen}{Salimans et~al\mbox{.}}{2016}]%
        {Salimans2016}
\bibfield{author}{\bibinfo{person}{Tim Salimans}, \bibinfo{person}{Ian
  Goodfellow}, \bibinfo{person}{Wojciech Zaremba}, \bibinfo{person}{Vicki
  Cheung}, \bibinfo{person}{Alec Radford}, \bibinfo{person}{Xi Chen}, {and}
  \bibinfo{person}{Xi Chen}.} \bibinfo{year}{2016}\natexlab{}.
\newblock \showarticletitle{Improved Techniques for Training GANs}. In
  \bibinfo{booktitle}{\emph{Advances in Neural Information Processing
  Systems}}, \bibfield{editor}{\bibinfo{person}{D.~Lee},
  \bibinfo{person}{M.~Sugiyama}, \bibinfo{person}{U.~Luxburg},
  \bibinfo{person}{I.~Guyon}, {and} \bibinfo{person}{R.~Garnett}} (Eds.),
  Vol.~\bibinfo{volume}{29}. \bibinfo{publisher}{Curran Associates, Inc.}
\newblock
\urldef\tempurl%
\url{https://proceedings.neurips.cc/paper/2016/file/8a3363abe792db2d8761d6403605aeb7-Paper.pdf}
\showURL{%
\tempurl}


\bibitem[\protect\citeauthoryear{Sarmad, Lee, and Kim}{Sarmad
  et~al\mbox{.}}{2019}]%
        {Sarmad2019}
\bibfield{author}{\bibinfo{person}{Muhammad Sarmad},
  \bibinfo{person}{Hyunjoo~Jenny Lee}, {and} \bibinfo{person}{Young~Min Kim}.}
  \bibinfo{year}{2019}\natexlab{}.
\newblock \showarticletitle{{RL-GAN-net: A reinforcement learning agent
  controlled gan network for real-time point cloud shape completion}}. In
  \bibinfo{booktitle}{\emph{IEEE CVPR}}.
\newblock
\urldef\tempurl%
\url{https://doi.org/10.1109/CVPR.2019.00605}
\showDOI{\tempurl}


\bibitem[\protect\citeauthoryear{Shu, Park, and Kwon}{Shu
  et~al\mbox{.}}{2019}]%
        {Shu2019}
\bibfield{author}{\bibinfo{person}{Dongwook Shu}, \bibinfo{person}{Sung~Woo
  Park}, {and} \bibinfo{person}{Junseok Kwon}.}
  \bibinfo{year}{2019}\natexlab{}.
\newblock \showarticletitle{{3D point cloud generative adversarial network
  based on tree structured graph convolutions}}. In
  \bibinfo{booktitle}{\emph{IEEE ICCV}}.
\newblock
\urldef\tempurl%
\url{https://doi.org/10.1109/ICCV.2019.00396}
\showDOI{\tempurl}


\bibitem[\protect\citeauthoryear{Torfi, Fox, and Reddy}{Torfi
  et~al\mbox{.}}{2020}]%
        {Torfi2020}
\bibfield{author}{\bibinfo{person}{Amirsina Torfi}, \bibinfo{person}{Edward~A.
  Fox}, {and} \bibinfo{person}{Chandan~K. Reddy}.}
  \bibinfo{year}{2020}\natexlab{}.
\newblock \bibinfo{title}{Differentially Private Synthetic Medical Data
  Generation using Convolutional GANs}.
\newblock
\newblock
\showeprint[arxiv]{2012.11774}


\bibitem[\protect\citeauthoryear{Torkzadehmahani, Kairouz, and
  Paten}{Torkzadehmahani et~al\mbox{.}}{2019}]%
        {Torkzadehmahani2019}
\bibfield{author}{\bibinfo{person}{Reihaneh Torkzadehmahani},
  \bibinfo{person}{Peter Kairouz}, {and} \bibinfo{person}{Benedict Paten}.}
  \bibinfo{year}{2019}\natexlab{}.
\newblock \showarticletitle{DP-CGAN: Differentially Private Synthetic Data and
  Label Generation}. In \bibinfo{booktitle}{\emph{IEEE CVPR}}.
\newblock


\bibitem[\protect\citeauthoryear{Torlak}{Torlak}{2012}]%
        {Torlak2012}
\bibfield{author}{\bibinfo{person}{Emina Torlak}.}
  \bibinfo{year}{2012}\natexlab{}.
\newblock \showarticletitle{Scalable Test Data Generation from Multidimensional
  Models}. In \bibinfo{booktitle}{\emph{ACM SIGSOFT}}.
\newblock
\urldef\tempurl%
\url{https://doi.org/10.1145/2393596.2393637}
\showDOI{\tempurl}


\bibitem[\protect\citeauthoryear{Vel{\'{a}}zquez, Mart{\'{i}}nez, Getzin,
  Moloney, and Wiegand}{Vel{\'{a}}zquez et~al\mbox{.}}{2016}]%
        {Velazquez2016}
\bibfield{author}{\bibinfo{person}{Eduardo Vel{\'{a}}zquez},
  \bibinfo{person}{Isabel Mart{\'{i}}nez}, \bibinfo{person}{Stephan Getzin},
  \bibinfo{person}{Kirk~A. Moloney}, {and} \bibinfo{person}{Thorsten Wiegand}.}
  \bibinfo{year}{2016}\natexlab{}.
\newblock \bibinfo{title}{{An evaluation of the state of spatial point pattern
  analysis in ecology}}.
\newblock
\newblock
\showISSN{16000587}
\urldef\tempurl%
\url{https://doi.org/10.1111/ecog.01579}
\showDOI{\tempurl}


\bibitem[\protect\citeauthoryear{Wang and Xu}{Wang and Xu}{2019}]%
        {Wang2019b}
\bibfield{author}{\bibinfo{person}{Di Wang} {and} \bibinfo{person}{Jinhui Xu}.}
  \bibinfo{year}{2019}\natexlab{}.
\newblock \showarticletitle{On Sparse Linear Regression in the Local
  Differential Privacy Model}. In \bibinfo{booktitle}{\emph{ICML}},
  Vol.~\bibinfo{volume}{97}. \bibinfo{publisher}{PMLR},
  \bibinfo{pages}{6628--6637}.
\newblock
\urldef\tempurl%
\url{https://proceedings.mlr.press/v97/wang19m.html}
\showURL{%
\tempurl}


\bibitem[\protect\citeauthoryear{Wang, Rudolph, Nepal, Grobler, and Chen}{Wang
  et~al\mbox{.}}{2020}]%
        {Wang2020}
\bibfield{author}{\bibinfo{person}{Shuo Wang}, \bibinfo{person}{Carsten
  Rudolph}, \bibinfo{person}{Surya Nepal}, \bibinfo{person}{Marthie Grobler},
  {and} \bibinfo{person}{Shangyu Chen}.} \bibinfo{year}{2020}\natexlab{}.
\newblock \showarticletitle{PART-GAN: Privacy-Preserving Time-Series Sharing}.
  In \bibinfo{booktitle}{\emph{Artificial Neural Networks and Machine
  Learning}}. \bibinfo{pages}{578--593}.
\newblock


\bibitem[\protect\citeauthoryear{Warner}{Warner}{1965}]%
        {Warner1965}
\bibfield{author}{\bibinfo{person}{Stanley~L. Warner}.}
  \bibinfo{year}{1965}\natexlab{}.
\newblock \showarticletitle{Randomized Response: A Survey Technique for
  Eliminating Evasive Answer Bias}.
\newblock \bibinfo{journal}{\emph{J. Amer. Statist. Assoc.}}
  \bibinfo{volume}{60}, \bibinfo{number}{309} (\bibinfo{year}{1965}),
  \bibinfo{pages}{63--69}.
\newblock
\urldef\tempurl%
\url{https://doi.org/10.1080/01621459.1965.10480775}
\showDOI{\tempurl}


\bibitem[\protect\citeauthoryear{Xiao, Farajtabar, Ye, Yan, Song, and Zha}{Xiao
  et~al\mbox{.}}{2017}]%
        {Xiao2017}
\bibfield{author}{\bibinfo{person}{Shuai Xiao}, \bibinfo{person}{Mehrdad
  Farajtabar}, \bibinfo{person}{Xiaojing Ye}, \bibinfo{person}{Junchi Yan},
  \bibinfo{person}{Le Song}, {and} \bibinfo{person}{Hongyuan Zha}.}
  \bibinfo{year}{2017}\natexlab{}.
\newblock \showarticletitle{{Wasserstein Learning of Deep Generative Point
  Process Models}}. In \bibinfo{booktitle}{\emph{NeurIPS}}.
  \bibinfo{pages}{3247--3257}.
\newblock
\urldef\tempurl%
\url{http://papers.nips.cc/paper/6917-wasserstein-learning-of-deep-generative-point-process-models}
\showURL{%
\tempurl}


\bibitem[\protect\citeauthoryear{Xiao and Xiong}{Xiao and Xiong}{2015}]%
        {Xiao2015}
\bibfield{author}{\bibinfo{person}{Yonghui Xiao} {and} \bibinfo{person}{Li
  Xiong}.} \bibinfo{year}{2015}\natexlab{}.
\newblock \showarticletitle{Protecting Locations with Differential Privacy
  under Temporal Correlations}. In \bibinfo{booktitle}{\emph{ACM SIGSAC}}.
  \bibinfo{pages}{1298–1309}.
\newblock
\urldef\tempurl%
\url{https://doi.org/10.1145/2810103.2813640}
\showDOI{\tempurl}


\bibitem[\protect\citeauthoryear{Xie, Lin, Wang, Wang, and Zhou}{Xie
  et~al\mbox{.}}{2018}]%
        {Xie2018}
\bibfield{author}{\bibinfo{person}{Liyang Xie}, \bibinfo{person}{Kaixiang Lin},
  \bibinfo{person}{Shu Wang}, \bibinfo{person}{Fei Wang}, {and}
  \bibinfo{person}{Jiayu Zhou}.} \bibinfo{year}{2018}\natexlab{}.
\newblock \showarticletitle{Differentially Private Generative Adversarial
  Network}.
\newblock  (\bibinfo{year}{2018}).
\newblock
\showeprint[arxiv]{1802.06739}
\urldef\tempurl%
\url{http://arxiv.org/abs/1802.06739}
\showURL{%
\tempurl}


\bibitem[\protect\citeauthoryear{Xie, Wang, Wei, Wang, and Tian}{Xie
  et~al\mbox{.}}{2016}]%
        {Xie2016}
\bibfield{author}{\bibinfo{person}{Lingxi Xie}, \bibinfo{person}{Jingdong
  Wang}, \bibinfo{person}{Zhen Wei}, \bibinfo{person}{Meng Wang}, {and}
  \bibinfo{person}{Qi Tian}.} \bibinfo{year}{2016}\natexlab{}.
\newblock \showarticletitle{{DisturbLabel: Regularizing CNN on the loss
  layer}}. In \bibinfo{booktitle}{\emph{IEEE CVPR}}.
\newblock
\urldef\tempurl%
\url{https://doi.org/10.1109/CVPR.2016.514}
\showDOI{\tempurl}


\bibitem[\protect\citeauthoryear{Xiong, Liu, Li, Wang, and Niu}{Xiong
  et~al\mbox{.}}{2019}]%
        {Xiong2019}
\bibfield{author}{\bibinfo{person}{Xingxing Xiong}, \bibinfo{person}{Shubo
  Liu}, \bibinfo{person}{Dan Li}, \bibinfo{person}{Jun Wang}, {and}
  \bibinfo{person}{Xiaoguang Niu}.} \bibinfo{year}{2019}\natexlab{}.
\newblock \showarticletitle{Locally differentially private continuous location
  sharing with randomized response}.
\newblock \bibinfo{journal}{\emph{International Journal of Distributed Sensor
  Networks}} \bibinfo{volume}{15}, \bibinfo{number}{8} (\bibinfo{year}{2019}).
\newblock
\urldef\tempurl%
\url{https://doi.org/10.1177/1550147719870379}
\showDOI{\tempurl}


\bibitem[\protect\citeauthoryear{Yoon, Jordon, and van~der Schaar}{Yoon
  et~al\mbox{.}}{2019}]%
        {Yoon2018}
\bibfield{author}{\bibinfo{person}{Jinsung Yoon}, \bibinfo{person}{James
  Jordon}, {and} \bibinfo{person}{Mihaela van~der Schaar}.}
  \bibinfo{year}{2019}\natexlab{}.
\newblock \showarticletitle{{PATE}-{GAN}: Generating Synthetic Data with
  Differential Privacy Guarantees}. In \bibinfo{booktitle}{\emph{ICLR}}.
\newblock
\urldef\tempurl%
\url{https://openreview.net/forum?id=S1zk9iRqF7}
\showURL{%
\tempurl}


\bibitem[\protect\citeauthoryear{Yuan, Zheng, Xie, and Sun}{Yuan
  et~al\mbox{.}}{2011}]%
        {Yuan2011}
\bibfield{author}{\bibinfo{person}{Jing Yuan}, \bibinfo{person}{Yu Zheng},
  \bibinfo{person}{Xing Xie}, {and} \bibinfo{person}{Guangzhong Sun}.}
  \bibinfo{year}{2011}\natexlab{}.
\newblock \showarticletitle{Driving with knowledge from the physical world}. In
  \bibinfo{booktitle}{\emph{ACM SIGKDD}}. \bibinfo{pages}{316}.
\newblock
\urldef\tempurl%
\url{https://doi.org/10.1145/2020408.2020462}
\showDOI{\tempurl}


\bibitem[\protect\citeauthoryear{Yuan, Zheng, Zhang, Xie, Xie, Sun, and
  Huang}{Yuan et~al\mbox{.}}{2010}]%
        {Yuan2010}
\bibfield{author}{\bibinfo{person}{Jing Yuan}, \bibinfo{person}{Yu Zheng},
  \bibinfo{person}{Chengyang Zhang}, \bibinfo{person}{Wenlei Xie},
  \bibinfo{person}{Xing Xie}, \bibinfo{person}{Guangzhong Sun}, {and}
  \bibinfo{person}{Yan Huang}.} \bibinfo{year}{2010}\natexlab{}.
\newblock \showarticletitle{T-drive: driving directions based on taxi
  trajectories}. In \bibinfo{booktitle}{\emph{ACM SIGSPATIAL}}.
  \bibinfo{pages}{99}.
\newblock
\urldef\tempurl%
\url{https://doi.org/10.1145/1869790.1869807}
\showDOI{\tempurl}


\bibitem[\protect\citeauthoryear{Yuan, Shen, Mironov, and Nascimento}{Yuan
  et~al\mbox{.}}{2021}]%
        {Yuan2021}
\bibfield{author}{\bibinfo{person}{Sen Yuan}, \bibinfo{person}{Milan Shen},
  \bibinfo{person}{Ilya Mironov}, {and} \bibinfo{person}{Anderson C.~A.
  Nascimento}.} \bibinfo{year}{2021}\natexlab{}.
\newblock \bibinfo{title}{Practical, Label Private Deep Learning Training based
  on Secure Multiparty Computation and Differential Privacy}.
\newblock \bibinfo{howpublished}{Cryptology ePrint Archive, Report 2021/835}.
\newblock
\urldef\tempurl%
\url{https://ia.cr/2021/835}
\showURL{%
\tempurl}


\bibitem[\protect\citeauthoryear{Zhang, Cormode, Procopiuc, Srivastava, and
  Xiao}{Zhang et~al\mbox{.}}{2017}]%
        {Zhang2017}
\bibfield{author}{\bibinfo{person}{Jun Zhang}, \bibinfo{person}{Graham
  Cormode}, \bibinfo{person}{Cecilia~M. Procopiuc}, \bibinfo{person}{Divesh
  Srivastava}, {and} \bibinfo{person}{Xiaokui Xiao}.}
  \bibinfo{year}{2017}\natexlab{}.
\newblock \showarticletitle{PrivBayes: Private Data Release via Bayesian
  Networks}.
\newblock \bibinfo{journal}{\emph{ACM Trans. Database Syst.}}
  \bibinfo{volume}{42}, \bibinfo{number}{4} (\bibinfo{year}{2017}).
\newblock
\urldef\tempurl%
\url{https://doi.org/10.1145/3134428}
\showDOI{\tempurl}


\bibitem[\protect\citeauthoryear{Zhang, Li, Zhou, Kong, and Luo}{Zhang
  et~al\mbox{.}}{2020}]%
        {Zhang2020}
\bibfield{author}{\bibinfo{person}{Yingxue Zhang}, \bibinfo{person}{Yanhua Li},
  \bibinfo{person}{Xun Zhou}, \bibinfo{person}{Xiangnan Kong}, {and}
  \bibinfo{person}{Jun Luo}.} \bibinfo{year}{2020}\natexlab{}.
\newblock \showarticletitle{{Curb-GAN: Conditional Urban Traffic Estimation
  through Spatio-Temporal Generative Adversarial Networks}}. In
  \bibinfo{booktitle}{\emph{ACM SIGKDD}}.
\newblock
\urldef\tempurl%
\url{https://doi.org/10.1145/3394486.3403127}
\showDOI{\tempurl}


\bibitem[\protect\citeauthoryear{Zhang, Wang, Li, Honorio, Backes, He, Chen,
  and Zhang}{Zhang et~al\mbox{.}}{2021}]%
        {Zhang2021}
\bibfield{author}{\bibinfo{person}{Zhikun Zhang}, \bibinfo{person}{Tianhao
  Wang}, \bibinfo{person}{Ninghui Li}, \bibinfo{person}{Jean Honorio},
  \bibinfo{person}{Michael Backes}, \bibinfo{person}{Shibo He},
  \bibinfo{person}{Jiming Chen}, {and} \bibinfo{person}{Yang Zhang}.}
  \bibinfo{year}{2021}\natexlab{}.
\newblock \showarticletitle{PrivSyn: Differentially Private Data Synthesis}. In
  \bibinfo{booktitle}{\emph{{USENIX} Security Symposium}}.
\newblock
\urldef\tempurl%
\url{https://www.usenix.org/conference/usenixsecurity21/presentation/zhang-zhikun}
\showURL{%
\tempurl}


\bibitem[\protect\citeauthoryear{Zhou, Chen, Liao, Chen, Dong, Liu, Zhang, Hua,
  and Yu}{Zhou et~al\mbox{.}}{2020}]%
        {Zhou2020}
\bibfield{author}{\bibinfo{person}{Hang Zhou}, \bibinfo{person}{Dongdong Chen},
  \bibinfo{person}{Jing Liao}, \bibinfo{person}{Kejiang Chen},
  \bibinfo{person}{Xiaoyi Dong}, \bibinfo{person}{Kunlin Liu},
  \bibinfo{person}{Weiming Zhang}, \bibinfo{person}{Gang Hua}, {and}
  \bibinfo{person}{Nenghai Yu}.} \bibinfo{year}{2020}\natexlab{}.
\newblock \showarticletitle{{LG-GAN: Label Guided Adversarial Network for
  Flexible Targeted Attack of Point Cloud Based Deep Networks}}. In
  \bibinfo{booktitle}{\emph{IEEE CVPR}}.
\newblock
\urldef\tempurl%
\url{https://doi.org/10.1109/CVPR42600.2020.01037}
\showDOI{\tempurl}


\bibitem[\protect\citeauthoryear{Zuo, Agterberg, Cheng, and Yao}{Zuo
  et~al\mbox{.}}{2009}]%
        {Zuo2009}
\bibfield{author}{\bibinfo{person}{Renguang Zuo}, \bibinfo{person}{Frederik~P.
  Agterberg}, \bibinfo{person}{Qiuming Cheng}, {and} \bibinfo{person}{Lingqing
  Yao}.} \bibinfo{year}{2009}\natexlab{}.
\newblock \showarticletitle{{Fractal characterization of the spatial
  distribution of geological point processes}}.
\newblock \bibinfo{journal}{\emph{International Journal of Applied Earth
  Observation and Geoinformation}} (\bibinfo{year}{2009}).
\newblock
\urldef\tempurl%
\url{https://doi.org/10.1016/j.jag.2009.07.001}
\showDOI{\tempurl}


\end{thebibliography}

\balance

\end{document}